\documentclass[letterpaper, 10 pt, conference]{ieeeconf}  % Comment this line out if you need a4paper

\IEEEoverridecommandlockouts                              % This command is only needed if 
\usepackage[T1]{fontenc}
\usepackage{preamble}

\title{\LARGE \bf
Line-of-Sight-Constrained Multi-Robot Mapless Navigation via Polygonal Visible Region
}

\author{Ruofei Bai, Shenghai Yuan, Xinhang Xu, Xingyu Ji, Xiaowei Li, Hongliang Guo, Wei-Yun Yau, Lihua Xie% <-this % stops a space
}

\begin{document}

\maketitle
\thispagestyle{empty}
% \pagestyle{empty}
% 启用页码显示
% \pagestyle{plain}

%%%%%%%%%%%%%%%%%%%%%%%%%%%%%%%%%%%%%%%%%%%%%%%%%%%%%%%%%%%%%%%%%%%%%%%%%%%%%%%%
\begin{abstract}
Multi-robot systems rely on underlying connectivity to ensure reliable communication and timely coordination.
This paper studies the line-of-sight (LoS) connectivity maintenance problem in multi-robot navigation with unknown obstacles.
Prior works typically assume known environment maps to formulate LoS constraints between robots, which hinders their practical deployment.
To overcome this limitation, we propose an inherently distributed approach where each robot only constructs an egocentric visible region based on its real-time LiDAR scans, instead of endeavoring to build a global map online.
The individual visible regions are shared through distributed communication to establish inter-robot LoS constraints, which are then incorporated into a multi-robot navigation framework to ensure LoS-connectivity.
% Specifically, we propose a more reliable and efficient LoS-distance metric than previously used ones based on the polygonal approximation of visible regions.
% Moreover, we enhance the robustness of connectivity maintenance by proposing a more accurate LoS-distance metric, which further enables flexible topology optimization to eliminate redundant connections that require substantial effort to maintain under external navigation tasks.
Moreover, we enhance the robustness of connectivity maintenance by proposing a more accurate LoS-distance metric, which further enables network topology optimization by eliminating redundant and effort-demanding connections for improved navigation efficiency.
% Moreover, we achieve robustness LoS-connectivity maintenance by designing more accurate LoS-distance evaluation metric, and propose topology optimization that flexibily omits redundant and effort-demanding connections to improve overall navigation efficiency.
% Specifically, we propose the polygonal approximation of visible regions that supports reliable and efficient LoS-distance evaluation between robots, which further enables us to optimize the connectivity topology by omitting redundant and effort-demanding connections to mitigate their adverse impact on navigation efficiency.
% To minimize the effect of connectivity constraints to navigation efficiency, we optimize the connectivity topology by omitting redundant and effort-demanding connections, which can be directly implemented by masking out corresponding edges from the weighted graph Laplacian matrix of robots' underlying network.
% The proposed framework is evaluated with extensive multi-robot navigation and exploration tasks in both simulation and real-world experiments, focusing on its robustness, efficiency and practicality.
The proposed framework is evaluated through extensive multi-robot navigation and exploration tasks in both simulation and real-world experiments. 
Results show that it reliably maintains LoS-connectivity between robots in challenging environments cluttered with obstacles, even under large visible ranges and fragile minimal topologies, where existing methods consistently fail.
Ablation studies also reveal that topology optimization boosts navigation efficiency by around 20\%, demonstrating the framework’s potential for efficient navigation under connectivity constraints.

\end{abstract}

%%%%%%%%%%%%%%%%%%%%%%%%%%%%%%%%%%%%%%%%%%%%%%%%%%%%%%%%%%%%%%%%%%%%%%%%%%%%%%%%

\section{Introduction}

Multi-robot navigation holds numerous applications in search and rescue~\cite{tian2020search}, collaborative inspection~\cite{cao2025cooperative}, and autonomous exploration~\cite{10577228}.
While individual robots have limited onboard communication and sensing capabilities, multi-robot systems can overcome such limitations and operate in wider areas.
To ensure reliable inter-robot communication and timely coordination, it is essential to maintain a connected network between robots~\cite{amigoni_MultirobotExploration_2017}.
% Due to onboard resource limitations, individual robots usually have limited communication and sensing capabilities.
% Instead, multi-robot systems can overcome such limitations and operate in wider areas by maintaining a connected network that ensures reliable inter-robot communication and timely coordination~\cite{amigoni_MultirobotExploration_2017}.
% While previous work has studied multi-robot connectivity maintenance problem considering limited communication ranges~\cite{shi_CommunicationAwareMultirobot_2021}, 
However, existing methods usually omit the presence of obstacles that degrade signals and prevent mutual observation between robots~\cite{shi_CommunicationAwareMultirobot_2021}, thereby compromising information sharing and mutual observation~\cite{liu_RelativeLocalizationEstimation_2023}.
% \red{For example, UWB-based localization system can be disrupted by obstacles.}
To this end, this paper considers a situation where each robot in a team must be within at least one of its neighbors' \emph{line-of-sight} (LoS) and communication range to be connected, while avoiding collisions, as shown in Fig.~\ref{fig_wide_grid}.
The problem is challenging as robots need to maintain LoS-connectivity in the presence of obstacles while performing external navigation tasks~\cite{amigoni_MultirobotExploration_2017}.

\begin{figure}[!t]
\vspace{6pt}
\centering
\includegraphics[width=\linewidth]{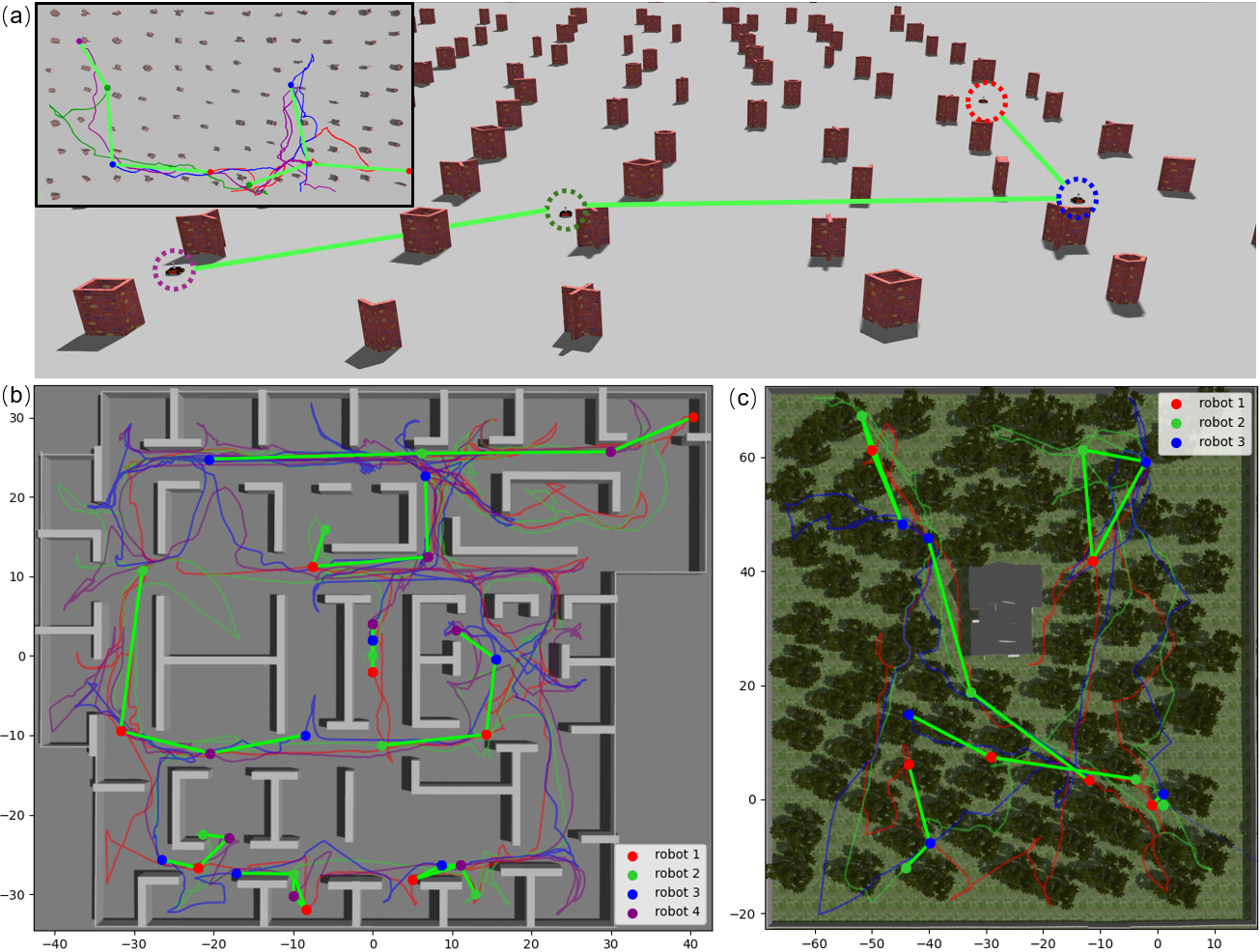}
\vspace{-20pt}
\caption{(a) Four robots navigate in a simulation environment cluttered with small irregular obstacles while always maintaining LoS-connectivity ($100m\times 50m$). (b) The trajectories of four robots and snapshots of their connectivity graphs when exploring an initially unknown garage environment~\cite{10577228} ($87m\times 69m$). (c) Three robots exploring an outdoor forest environment adapted from~\cite{cao2023representation} ($86m\times 95m$).}
\label{fig_wide_grid}
\vspace{-10pt}
\end{figure}

Prior studies on related problems typically assume known environmental models~\cite{amigoni_MultirobotExploration_2017, xia_RELINKRealTime_2023, robuffogiordano_PassivitybasedDecentralized_2013, nestmeyer_DecentralizedSimultaneous_2017}, where obstacles are either represented by a predefined set of obstacle points~\cite{robuffogiordano_PassivitybasedDecentralized_2013, amigoni_MultirobotExploration_2017, chen_MultiUAVDeployment_2023}, or a shared occupancy map between robots~\cite{shi_CommunicationAwareMultirobot_2021}.
% For example, traditional connectivity control methods usually assume a predefined set of obstacle points, based on which potential functions can be designed to keep the line-of-sight between two robots away from obstacle points~\cite{robuffogiordano_PassivitybasedDecentralized_2013, amigoni_MultirobotExploration_2017, chen_MultiUAVDeployment_2023}.
% Besides that, the LoS-connectivity between robots is usually verified based on shared occupancy maps of environments in discrete decision-making methods~\cite{shi_CommunicationAwareMultirobot_2021}.
However, reliance on prior maps prevents above methods from real-world applications, where the environment can be unknown with unpredictable and irregular obstacles that obstruct LoS connectivity between robots.
% However, the reliance on prior maps prevents existing works from real-world applications, where the presence of unpredictable and irregular obstacles can obstruct LoS connectivity between robots.
Moreover, it remains a non-trivial task to collaboratively build a map among robots to verify their LoS-connectivity during online operation.

% These existing works usually assume known environment models, like occupancy maps~\cite{xia_RELINKRealTime_2023} or obstacle point sets~\cite{robuffogiordano_PassivitybasedDecentralized_2013, nestmeyer_DecentralizedSimultaneous_2017}, for potential function design, or search for future target topology for robots.
% Moreover, in unknown environments with risks and complications, robots may have no access to obstacle information, in which case existing methods are not applicable.

% \begin{figure}[!t]
% \centering
% \includegraphics[width=\linewidth]{image/fig_face3.png}
% \caption{Real-time LoS maintenance of four robots while exploring an unknown environment. Robots' visible regions (enclosed by curves with different colors) are constructed from their real-time point cloud measurements.
% The connectivity between robots is shown by the light green edges. To maintain connectivity, each robot only needs local information from its one-hop neighbors.}
% \label{fig_face}
% \end{figure}

To eliminate the requirements on prior maps, we propose a novel idea that, instead of endeavoring to construct a shared global map, each robot only needs to explicitly construct an egocentric \emph{visible region} based on its real-time LiDAR scans.
The visible region describes the space that is currently visible from a robot, which is shared between neighboring robots to formulate their inter-robot LoS constraints.
The LoS constraints are then integrated into a graph Laplacian-based controller~\cite{robuffogiordano_PassivitybasedDecentralized_2013} for connectivity maintenance. 
% Specifically, we propose a polygonal approximation of the visible region that supports differentiable, efficient, and safe LoS-distance evaluation, which is integrated into a potential function to keep robots within others' visible regions to be connected.
% Unlike maintaining a centralized global map, the approximated visible region holds a unified analytical representation that can be conveniently shared among robots for LoS evaluation, inherently supporting distributed computation and communication.
% we propose an efficient method to establish LoS constraints directly from robots' real-time LiDAR scans.
% While raw LiDAR scans can include empty cavities or gaps due to its limited resolution, our method explicitly construct an enclosed visible region of each robot that supports setting different degrees of aggressiveness.
% Based on that, a differentiable LoS-distance metric is proposed and 
% For connectivity maintenance,
Although preliminary work is proposed in~\cite{bai2025realmrealtimelineofsightmaintenance}, it suffers from brittle connectivity due to the inaccurate LoS-distance metric and inefficient navigation caused by redundant connections.
Motivated by these limitations, we make substantial modifications to achieve more robust LoS-connectivity maintenance and flexible navigation behaviors. 
In particular, we reformulate a polygonal approximation of visible regions to support accurate LoS-distance evaluation, and propose an online topology optimization approach that facilitates robots' navigation efficiency.
% in several aspects. 
% First, we propose a polygonal approximation of visible regions that supports more accurate, efficient, and safe LoS-distance evaluation compared with previous methods.
% The metric holds an analytical expression and is proved to be a reliable lower-bound approximation of the actual LoS distance, ensuring a conservative assessment even under aggressive visible region formulation.
% Second, we propose a topology optimization approach that identifies and masks out redundant and effort-demanding connections, which facilitates robots' navigation efficiency.
Consequently, our method can reliably maintain global connectivity of robots even with fragile topologies and large visible ranges, while previous methods fail under such conditions. 
Our contributions are summarized as follows:

\begin{enumerate}
    \item We propose a polygonal approximation of a robot' visible region constructed from real-time LiDAR scans, which supports an efficient and reliable metric for LoS-distance evaluation between robots;
    \item We prove that the proposed metric lower-bounds the actual LoS distance, serving as a safe and conservative approximation regardless of the coverage range of visible regions.
    \item We propose a topology optimization approach in graph Laplacian-based controllers, where the network topology is flexibly adapted to minimize the efforts for connectivity maintenance between robots considering external navigation tasks.
    % \item \red{[TODO] Visibility-aware path planning with preference on neighboring robot's visible region, combining connectivity maintenance and target navigation for improved navigation efficiency.}
\end{enumerate}

The robustness and efficiency of the proposed method is extensively evaluated with multi-robot navigation and exploration tasks in unknown cluttered environments, where robots can successfully navigate to their targets while robustly maintaining LoS-connectivity. We also validate the applicability of the proposed framework in real-world experiments.

\section{Related Works}

\subsection{Line-of-Sight-Constrained Multi-Robot Navigation}

Existing work considering LoS-connectivity maintenance can be categorized into continuous and discrete types~\cite{amigoni_MultirobotExploration_2017}, where LoS constraints are usually formulated based on known global maps, such as obstacle points~\cite{robuffogiordano_PassivitybasedDecentralized_2013, yang_MinimallyConstrained_2023}, occupancy maps~\cite{stump_VisibilitybasedDeployment_2011, xia_RELINKRealTime_2023}, etc.
% Continuous connectivity requires robots to always maintain a connected network during their operations.
Giordano~\etal~proposed a potential function-based method that guarantees continuous connectivity by preserving the positivity of the Fiedler eigenvalue~\cite{robuffogiordano_PassivitybasedDecentralized_2013, amigoni_MultirobotExploration_2017}.
The LoS constraints are captured by the distance from the closest obstacle point to the LoS segment joining two robots.
% Alternatively, the LoS segment is approximated with a minimum volume enclosing ellipsoid to analytically formulate LoS constraints in~\cite{yang_MinimallyConstrained_2023}.
% They further consider the estimation uncertainty in robot positions in~\cite{yang_DecentralizedMultirobot_2024}. 
Chen~\etal~consider the LoS constraints in multi-UAV deployment by restricting two robots to be within a common separating hyperplane of obstacles~\cite{chen_MultiUAVDeployment_2023}.
In contrast, discrete connectivity only requires robots to be connected at certain time steps~\cite{hollinger_MultirobotCoordination_2012}.
Typically, a \emph{target topology} is first determined to ensure connectivity and then optimized to consider additional objectives, such as distance or information gain~\cite{shi_CommunicationAwareMultirobot_2021, stump_VisibilitybasedDeployment_2011, dutta_MultirobotInformative_2019, xia_RELINKRealTime_2023}.
% Stump~\etal~apply polygonal decomposition of an environment, after which local optimization of robots' positions is restricted in a common polyhedron to ensure LoS~\cite{stump_VisibilitybasedDeployment_2011}.
Xia~\etal~formulate the LoS constraints based on robots' visible regions described by star convex polytopes~\cite{xia_RELINKRealTime_2023}.
% However, local optimization of robots' positions may lead to potential LoS loss.
However, as \emph{local optimization} modifies robots' positions, the corresponding changes in their visible regions may lead to potential loss of LoS after optimization.
% Li~\etal~propose a reinforcement learning-based approach for connectivity maintenance based directly on LiDAR measurements~\cite{li_DecentralizedGlobal_2022}.
% However, the method omits LoS constraints and has no connectivity guarantee. 

As a preliminary attempt to alleviate the reliance on prior maps, Bai~\etal~mitigate the limitation of~\cite{xia_RELINKRealTime_2023} by applying the concept of visible region (constructed from onboard LiDAR scans) to continuous connectivity maintenance~\cite{bai2025realmrealtimelineofsightmaintenance}.
However, it suffers from brittle connectivity due to the inaccurate LoS-distance metric and inefficient navigation.
% While they modify the previous LoS-distance metric to encode both the urgency and sensitivity of losing LoS, the resulting method only works with visible regions that have conservative coverage.
% The reason is that both the original metric and its derived forms yield increasing errors as the sensing range increases.
% They also modify the original LoS-distance metric defined over the visible region with an angle-related coefficient to encode both the urgency and sensitivity of losing LoS. 
% However, it can only apply to visible regions with limited sensing ranges, because both the original LoS-distance metric and its derivatives fail to provide accurate LoS-distance evaluation, whose errors become prominent as the sensing ranges increase.
The above limitations are addressed in this work through a more reliable metric that supports robust LoS maintenance, which further enables flexible topology optimization for efficient robot navigation.
% Motivated by these limitations, this work reformulates an efficient and reliable LoS-distance metric that supports robust LoS maintenance and allows flexible connectivity topology of robots for efficient navigation performance.
% In this paper, we identify the reason behind this is that both the original LoS-distance metric and its derivatives fail to provide accurate LoS-distance evaluation, whose errors become prominent as the sensing ranges increase. 
% Such a limitation not only confines robots' operation ranges, but also can occasionally lead to connectivity loss in practical discrete control. 

\subsection{Topology Optimization in Connectivity Maintenance}

Ensuring team connectivity often imposes constraints on robot navigation.
To alleviate this, previous studies have proposed maintaining only the necessary connections among robots rather than preserving all of them~\cite{xia_RELINKRealTime_2023, shi_CommunicationAwareMultirobot_2021}. 
% The connectivity topology is optimized to reduce its effect to external navigation tasks.
% With prior maps, robots' topology can be directly optimized by searching for a minimum spanning tree against the occupancy maps.
Shi~\etal~studied the topology optimization in a communication-constrained multi-robot exploration problem~\cite{shi_CommunicationAwareMultirobot_2021}, where robots' future positions are first selected to maximize coverage, and then refined with minimum deviations to form a connected minimum spanning tree (MST).  
Xie~\etal~studied the problem of deploying the minimum number of robots to relay messages between a base station and several client points~\cite{xia_RELINKRealTime_2023}.
A Steiner Tree is optimized to determine the positions of relay robots that connect the base station to the target clients.
Similar to our work, Yang~\etal~studied the minimally LoS-connectivity constrained multi-robot coordination problem, where only a selected minimum spanning tree is maintained with control barrier functions~\cite{yang_MinimallyConstrained_2023, luo_BehaviorMixing_2020}.
% Similar to our works, Yang~\etal~studied the minimally LoS-connectivity constrained multi-robot coordination problem, where only a selected minimum spanning tree against robots' connectivity network is maintained, and the connections have less violation between the connectivity maintenance and external navigational movement~\cite{yang_MinimallyConstrained_2023, luo_BehaviorMixing_2020}.
% The connectivity of robots is guaranteed by designing control barrier functions based on known obstacle points in the environment.
Inspired by this work, we introduce a convenient topology optimization method within a graph Laplacian-based connectivity controller, which prunes unnecessary connections that require substantial effort to maintain under external navigation tasks, thereby improving navigation efficiency.

\section{Preliminaries}

\begin{figure*}[!t]
\vspace{6pt}
\centering
\includegraphics[width=\linewidth]{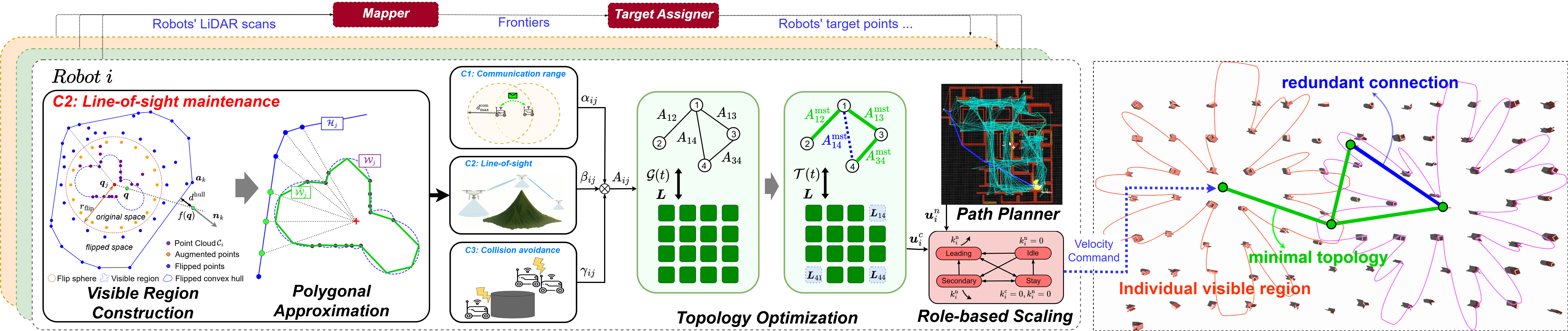}
\vspace{-18pt} 
\caption{Framework of LoS-connectivity constrained multi-robot navigation method.
Each robot independently constructs its visible region with the polygonal approximation, based on real-time LiDAR scans.
After sharing with immediate neighbors, three types of connectivity constraints are formulated as the weight of the robots' connectivity graph, which is then proceeded with topology optimization.
Next, the masked graph Laplacian matrix is used to derive the connectivity velocity $\boldsymbol{u}_{i}^{c}$, which is then fused with external navigational velocity $\boldsymbol{u}_{i}^{n}$ through the role-based scaling module, following the design in~\cite{bai2025realmrealtimelineofsightmaintenance}. 
Finally, each robot is driven independently by the fused velocity command to maintain LoS-connectivity between robots while navigating to its own target.
The navigational velocity is obtained from a mapless path planner proposed in~\cite{yang_FARPlanner_2022}.
The framework can also be extended for multi-robot exploration by integrating a multi-robot mapper for frontier extraction and a navigation-target assigner, as exemplified in Fig.~\ref{fig_wide_grid}(b)(c).
% The diagram particularly illustrates key modules including visible region construction~\cite{bai2025realmrealtimelineofsightmaintenance}, radial angle-based polygonal approximation, and connectivity topology optimization.
% The role-based scaling module aims to solve navigational conflicts and follows the design in~\cite{bai2025realmrealtimelineofsightmaintenance}. 
}
\label{fig_framework}
\vspace{-10pt}
\end{figure*}

\subsection{Graph Laplacian and Generalized Graph Connectivity}

Given a graph $\G = \langle \V, \E \rangle $ with $N$ vertices, the graph Laplacian matrix of $\G$ is defined as $\Lp = D - A$, where $A\in \mathbb{R}^{N\times N}$ is the adjacency matrix with an element $A_{ij} = 0$ if $(i, j)\notin \E$, and $A_{ij} > 0$ otherwise; $D\in \mathbb{R}^{N\times N}$ is a diagonal degree matrix where $D_{ii} = \sum_{j = 1}^{N}A_{ij}$ and $D_{ij}= 0$ when $i\ne j$. 
% As $\Lp$ is positive semi-definite, the $N$ eigenvalues of $\Lp$ are non-negative. 
% The $N$ eigenvalues of $\Lp$ can be sorted as $0= \lambda_1 \le \lambda_2 \le ... \le \lambda_{N}$, where the smallest eigenvalue $\lambda_1$ is equal to zero by construction. 
The second-smallest eigenvalue $\lambda_2$ of $\Lp$, usually referred to as the Fiedler eigenvalue~\cite{fiedler1973algebraic}, can be used to check the connectivity of $\G$.
It holds the property that $\lambda_2 > 0$ if the graph $\G$ is connected;
otherwise, $\lambda_2 = 0$ if $\G$ is disconnected.

By encoding the satisfaction of inter-robot connectivity constraints into the edge weight $A_{ij}$, the Fiedler eigenvalue $\lambda_2$ of the \emph{weighted} graph Laplacian matrix of robots' underlying connectivity graph can reflect the \emph{generalized} graph connectivity~\cite{robuffogiordano_PassivitybasedDecentralized_2013}.
As in~\cite{robuffogiordano_PassivitybasedDecentralized_2013,nestmeyer_DecentralizedSimultaneous_2017}, after defining a potential function $V^{\lambda}(\lambda_2) = \frac{1}{\lambda_2 - \lambda_{2}^{\text{min}}}$ to enforce the Fiedler eigenvalue $\lambda_2$ being larger than a preferred lower bound $\lambda_{2}^{\text{min}} >0$, 
the gradient direction for a robot $i\in \mathcal{R}$ to maintain the connectivity of $\mathcal{G}$ can be derived as~\cite{yang_DecentralizedEstimation_2010}
\begin{equation}
\small
    \boldsymbol{u}^{\text{c}}_i = -\frac{\partial V^{\lambda}(\lambda_2)}{\partial \lambda_2} \cdot \frac{\partial \lambda_2}{\partial \q_{i}} = -\frac{\partial V^{\lambda}(\lambda_2)}{\partial \lambda_2}\cdot \sum_{j\in \mathcal{N}_{i}} \frac{\partial A_{ij}}{\partial \q_i}(v_{2_i} - v_{2_j})^{2},
\label{eq_connectivity_force}
\end{equation}
where $v_{2_i}$ and $v_{2_j}$ are the $i$-th and $j$-th elements of the normalized eigenvector $\boldsymbol{v}_2$ corresponding to $\lambda_{2}$, respectively. 
Following the above concepts, this work ensures LoS-connectivity between robots by encoding connectivity constraints into the weighted graph Laplacian matrix of robots' underlying graphs.

\subsection{Visible Region Construction from Point Cloud}
\label{sec_preliminary_visibility}

% \begin{figure}[t]
% % \vspace{6pt}
% \centering
% \subfloat[]{\includegraphics[width=.5\linewidth]{image/fig_visible_regions.png}}\hspace{1pt}
% \subfloat[]{\includegraphics[width=.40\linewidth]{image/fig_stuck5.png}}
% \vspace{-5pt} 
% \caption{(a) Various aggressiveness of visible regions constructed with different $r_{\text{flip}}$ (m), where the robot (red cross) is surrounded by numerous small obstacles. 
% (b) Four robots get stuck around an obstacle when using a graph Laplacian-based connectivity controller, although the preferred $\lambda_{2}^{\text{min}}$ is set to a small value of $0.01$. 
% Without priority among these edges, the connectivity force traps robots in a livelock state.
% Note that the color of a connection varies from green to red as the connection becomes weak.}
% \label{fig_r_flip}
% \vspace{-15pt}
% \end{figure}

% \begin{figure}[!t]
% \centering
% \includegraphics[width=.9\linewidth]{image/fig_stuck2.png}
% \caption{}
% \label{fig_stuck_obstacle}
% \end{figure}

% While previous works rely on prior known maps for LoS maintenance, they cannot be directly applied in unknown environments.
This work eliminates the requirements of prior known maps by constructing a real-time egocentric visible region for individual robot based on its onboard LiDAR scans.
The process is adapted from visibility analysis techniques in computer graphics~\cite{katz2007direct}.
As shown in Fig.~\ref{fig_framework}, given a point cloud $\mathcal{C}_{j}\subseteq \mathbb{R}^{3}$ of a robot $j$, its visible region can be constructed with the following four steps~\cite{katz2007direct}:
% The process of , based on the visibility analysis techniques in computer graphics~\cite{katz2007direct}.
% The visible region provides an efficient way to check whether two robots are visible to each other, and more importantly, the distance to the boundary of the opponent's visible region.
% Unlike in a known environment, when robots navigate in an unknown environment, they can only rely on onboard sensor measurements to determine whether their neighbors are within their LoS.
% The process to derive the visible region of a robot $i\in \mathcal{R}$ based on its current point cloud measurement $\mathcal{H}_{i}$ includes following steps~\cite{katz2007direct, liu_StarconvexConstrained_2022, xia_RELINKRealTime_2023}:
1) Augmentation. Add augmented points to fill the gaps in the point cloud $\mathcal{C}_j$;
2) Spherical flipping. Flip each point $\q\in \mathcal{C}_j$ to a point $\q' = f(\q)$ located outside a sphere with a predefined radius $r_{\text{flip}}$, where $r_{\text{flip}}>\Vert \q \Vert$ and $f(\cdot)$ is the spherical flipping function defined as 
\begin{equation}
    f:\q \xrightarrow{} 2r_{\text{flip}}\cdot \frac{\q}{\|\q\|} - \q.
\end{equation}
The flipped point cloud is denoted as $\mathcal{C}_{j}'$;
% following the equation $\q' = 2r_{\text{flip}}\cdot \frac{\q}{\|\q\|} - \q$. 
3) Convex hull construction. Generate a convex hull of $\mathcal{C}_{j}'$, denoted as $\mathcal{H}_{j}$;
% 3) Convex hull construction. Generate a convex hull of $\mathcal{C}_{j}'$, denoted as $\mathcal{H}_{j} = \{\q\in \mathbb{R}^{3} \vert \boldsymbol{A}\q \preceq \boldsymbol{b}\}$, where $\boldsymbol{A} = [\begin{matrix}
%     \n_{1}, & \cdots ,& \n_{K}
% \end{matrix}]^{\top} \in \mathbb{R}^{K\times 3}$ with 
% $\n_k$ being the outward normal vector of the $k$-th face of $\mathcal{H}_{j}$; 
% and $\boldsymbol{b} = [\begin{matrix}
%     \n_{1}^{\top}\boldsymbol{a}_1, & \cdots, & \n_{K}^{\top}\boldsymbol{a}_K
% \end{matrix}]\in \mathbb{R}^{K}$, with
% $\boldsymbol{a}_k\in \mathbb{R}^{3}$ being a vertex point on the $k$-th face. 
4) Inversion. Inverse the boundary of $\mathcal{H}_{j}$ by performing spherical flipping again. 
The inversed boundary encloses the visible region of robot $j$, denoted as $\mathcal{W}_j$.
By construction, a larger $r_{\text{flip}}$ determines a larger visible region~\cite{katz2007direct,xia_RELINKRealTime_2023}.

% \begin{enumerate}
%     \item Augmentation. Add augmented points to fill the gaps in the point cloud $\mathcal{H}_i$;
%     \item Spherical flipping. Map each point $\q\in \mathcal{H}_i$ along the ray originating from the robot $i$ to a point $\q'$ outside a sphere with radius $r$, following the equation $\q' = 2r\cdot \frac{\q}{\|\q\|} - \q$. The flipped point cloud is denoted as $\mathcal{H}_{i}'$;
%     \item Convex hull construction. Generate convex hull of $\mathcal{H}_{i}'$, defined as $\mathcal{H}(\mathcal{H}_{i}') = \{x\in \mathbb{R}^3 \vert Ax \preceq b\}$, where $A = [\n_{1}, ..., \n_{k}]^{\top}\in \mathbb{R}^{K\times 3}$ is composed by normal vectors of all $K$ faces of the convex hull, and $\boldsymbol{b} = [\n_{1}^{\top}\boldsymbol{a}_{1}, ..., \n_{k}^{\top}\boldsymbol{a}_{K}]^{\top}\in \mathbb{R}^{K}$ with $\boldsymbol{a}_k\in \mathbb{R}^{3}$ being an arbitrary point on the face $k\in [1, ..., K]$.
%     \item Inversion. Inverse the convex hull by performing spherical flipping again, and the inversed shape of the convex hull encloses the visible region of robot $i$.
% \end{enumerate}

\begin{definition}[Visible Region]
Given the point cloud $\mathcal{C}_{j}$ of a robot $j$, and its corresponding convex hull $\mathcal{H}_{j}$ with $K$ faces, the visible region is defined as
\[\mathcal{W}_j = \{\q \vert \max_{k\in [1, ..., K]}\{\n_{k}^{\top}(f(\q) - \boldsymbol{a}_k)\} > 0\}\subseteq \mathbb{R}^{3},\]
where $\n_k$ is the outward normal vector of the $k$-th face of $\mathcal{H}_{j}$, and $\boldsymbol{a}_k\in \mathbb{R}^{3}$ is a vertex point on the $k$-th face.
\label{def_visible_region}
\end{definition}

% By construction, a visible region is a star-convex polyhedron, \ie, for a visible region $\mathcal{W}_j$ centered at $\q_j$, $\forall \q\in \mathcal{W}_j$, $\forall \alpha \in [0, 1]$, $\alpha\q + (1-\alpha)(\q - \q_j) \in \mathcal{W}_j$.
% A larger $r_{\text{flip}}$ determines a more aggressive visible region~\cite{katz2007direct,xia_RELINKRealTime_2023}.

\begin{proposition}[Visibility Determination~\cite{katz2007direct}]
    A point $\q\in\mathbb{R}^{3}$ is within the LoS of a robot $j$ if $\q\in\mathcal{W}_j$; or equivalently, a determination metric $d^{\text{hull}} > 0$, where
    $d^{\text{hull}} = \max\{d_k = \n_{k}^{\top}(f(\q) - \boldsymbol{a}_k) \vert k = 1, ..., K\}$ for $\mathcal{H}_j$ with $K$ faces.
\label{prop_visible_determiniation}
\end{proposition}

% Whether a point $\q\in \mathbb{R}^{3}$ belongs to $\mathcal{W}_{j}$ is determined by a determination metric $d^{\text{hull}} = \max_{k}\{d_k = \n_{k}^{\top}(\q' - \boldsymbol{a}_k) \vert k = 1, ..., K\}$,
% where the index $k^{*} = \arg \max_{k\in [1,...,K]} d_{k}$.
% As shown in Fig.~\ref{fig_visibility_illustration}(a), whether a point $\q$ belongs to $\mathcal{W}_{j}$ is determined by the distance 
% $d^{\text{hull}} = \max_{k}\{d_k = \n_{k}^{\top}(\q' - \boldsymbol{a}_k) \vert k = 1, ..., K\}$,
% where the index $k^{*} = \arg \max_{k\in [1,...,K]} d_{k}$.
% If $d^{\text{hull}} > 0$, $\q$ is visible from $\q_j$, \ie, they are within line-of-sight with each other; vice verse if $d^{\text{hull}} \le 0$. 
% Note we denote $d^{\text{hull}}_{ji}$ as the determination metric of a robot $i$ to robot $j$'s $\mathcal{H}_{j}$.
A differentiable approximation of the determination metric $d^{\text{hull}}$ can be defined with the log-sum-exp relaxation as
\begin{equation}
    d^{\text{hull}} \approx \frac{1}{\alpha} \log \left(e^{\alpha d_{1}}+\cdots+e^{\alpha d_{K}}\right),
\label{eq_dx}
\end{equation}
where $\alpha > 0$ is a coefficient that controls the degree of approximation~\cite{liu_StarconvexConstrained_2022}.

\section{Problem Formulation}

\subsection{Robot Model}

Assume there are a set of robots $\mathcal{R} = \{1, 2, ..., R\}$ navigating in an initially unknown environment.
The position of a robot $i\in \mathcal{R}$ at time $t$ is denoted as $\q_{i}(t)\in \mathbb{R}^{3}$. 
Its kinematic model is given by
\begin{equation}
    \q_{i}(t+1) = \q_{i}(t) + \boldsymbol{u}_{i}(t),
\label{eq_kinematic_model}
\end{equation}
where $\boldsymbol{u}_{i}(t) = k^{\text{c}}_{i}\cdot \boldsymbol{u}^{\text{c}}_{i}(t) + k^{\text{n}}_{i}\cdot \boldsymbol{u}^{\text{n}}_{i}(t)$; 
$\boldsymbol{u}^{\text{c}}_{i}, \boldsymbol{u}^{\text{n}}_{i}\in \mathbb{R}^{2,3}$ are the velocity commands for connectivity and navigation, respectively; 
% \begin{equation}
%     \q_{i}(t+1) = \q_{i}(t) + k^{\text{c}}_{i}\cdot \boldsymbol{u}^{\text{c}}_{i}(t) + k^{\text{n}}_{i}\cdot \boldsymbol{u}^{\text{n}}_{i}(t),
% \label{eq_kinematic_model}
% \end{equation}
% where $\boldsymbol{u}^{\text{c}}_{i}, \boldsymbol{u}^{\text{n}}_{i}\in \mathbb{R}^{3}$ are the velocity commands for connectivity and navigation, respectively;  
and $k_i^\text{c}, k_i^\text{n}\in \mathbb{R}_{\ge 0}$ are two scaling factors.
Moreover, $\Vert \boldsymbol{u}_{i}(t) \Vert_2\le U_{\text{max}}$, where $U_{\text{max}}$ is the upper limit of robots' velocities.

\subsection{Connectivity Constraints}
\label{sec_types_of_constraints}

Let $\G(t) = \langle \V, \E(t), \omega \rangle$ be an undirected time-varying connectivity graph of robots, where $\V = \{1, ..., R\}$ includes $R$ vertices corresponding to robots in $\mathcal{R}$; $\E(t) \subseteq \V \times \V$ is the edge set at time $t$;
and the weighting function $\omega: \V \times \V \xrightarrow{} \mathbb{R}_{\ge 0}$ evaluates the strength of edge between two vertices in the graph.
An edge $(i, j)$ exists between two robots $i$ and $j$ if and only if the following constraints are satisfied, where $d_{ij}= \lVert \q_{i} - \q_{j} \rVert_{2}$.

\begin{itemize}
    \item (C1) Communication range. The distance $d_{ij}$ must be within a communication range $d^{\text{com}}_{\text{max}}$, \ie, $d_{ij}\le d^{\text{com}}_{\text{max}}$.
    \item (C2) Line-of-Sight maintenance. The two robots must be within each other's line-of-sight, \ie, $\eta \q_{i} + (1-\eta)\q_{j} \notin \mathcal{O}$, $\forall \eta \in [0, 1]$, where $\mathcal{O}\subseteq \mathbb{R}^{3}$ denotes the space occupied by obstacles.
    \item (C3) Collision avoidance. The inter-robot and robot-obstacle distances must be no less than a safe distance $d^{\text{coll}}_{\text{min}}$, \ie, $d_{ij}\ge d^{\text{coll}}_{\text{min}}$; $\min_{\q\in \mathcal{O}}\Vert \q_i - \q\Vert_{2} \ge d^{\text{coll}}_{\text{min}}$.
\end{itemize}
We define $\omega(\cdot)$ as $\omega(i, j) =A_{ij} = \alpha_{ij}\cdot \beta_{ij} \cdot \gamma_{ij}$, where $\alpha(\cdot), \beta(\cdot), \gamma(\cdot)$ are three potential functions to quantify the satisfaction of constraints C1, C2, and C3 between robots, respectively.
As the functions $\alpha(\cdot)$ and $\gamma(\cdot)$ only depend on robots' relative distances regardless of the environment, we define them following~\cite{bai2025realmrealtimelineofsightmaintenance}.
We instead focus on formulating $\beta(\cdot)$ in this paper.
% The resulting graph $\G$ is referred to as the \emph{connectivity graph} of the robot team. 
Initially, we assume $\G(0)$ is connected.

% We define the three potential functions 
% $\alpha(\cdot), \beta(\cdot), \gamma(\cdot)$ as in previous work~\cite{bai2025realmrealtimelineofsightmaintenance} to quantify the satisfaction of constraints C1, C2, and C3 between two robots $i$ and $j$, respectively.
% Based on the concept of generalized graph connectivity, we define $\omega(i, j) = A_{ij} = \alpha_{ij}\cdot \beta_{ij} \cdot \gamma_{ij}$ for an edge $(i, j)\in \mathcal{E}$ in $\mathcal{G}$.

% We define a set $\N_{i} = \{j\in \V \vert (i, j)\in \E\}$ as the collection of all neighboring robots of robot $i$ that satisfy above constraints.

\subsection{Problem Statement}
\label{eq_problem_statement}

Given a 
sequence of target points $\mathcal{Z} = \{\boldsymbol{z}^{1}, ..., \boldsymbol{z}^{M}\}\subseteq \mathbb{R}^{3}$ distributed in the free space of an unknown environment, and a group of robots $\mathcal{R} = \{1, ..., R\}$ with kinematic models defined in Eq.~(\ref{eq_kinematic_model}).
We assume $M \le R$, and a target $\boldsymbol{z}^m\in \mathcal{Z}$ is assigned to a robot $i^m\in \mathcal{R}$.
$\forall t\ge 0$, the problem is to find a sequence of velocity commands $\boldsymbol{u}_{i}(t)$ for each robot $i\in \mathcal{R}$ so that:
(1) there exists a time sequence $0\le t^1, ..., t^M < \infty$ when a target $\boldsymbol{z}^m$ is visited by the assigned robot $i^m$ at time $t^m$;
(2) the connectivity graph $\G(t)$ is always connected; 
(3) the time required to visit all targets is minimized.

% connectivity maintenance minimally interferes with robots' navigation efficiency.

The key challenge lies in reliably maintaining LoS connectivity between robots in unknown environments while minimizing the impact on their external navigation tasks, which is discussed in the following sections.

\section{Methodology}

This section presents our approach for multi-robot navigation with LoS connectivity maintenance.
The entire framework is shown in Fig.~\ref
{fig_framework}.
We first introduce a new metric for LoS-distance evaluation based on the polygonal visible region, and then propose a topology optimization approach to improve navigation efficiency.
% We first introduce an efficient and reliable LoS-distance metric based on the polygonal approximation of a robot's visible region, supporting robust LoS connectivity maintenance.
% The new metric significantly improves robustness of LoS maintenance, which then supports
% With such capability, we further propose an additional topology optimization module that masks out redundant connections from the graph Laplacian matrix, therefore facilitating robots' navigation efficiency while maintaining connectivity.
% Based on that, we propose topology optimization in graph Laplacian-based connectivity controllers to allow a more dispersed formation of robots under connectivity constraints.
We focus on the 2D cases for conciseness while the formulations also apply to 3D cases.
Unless specified otherwise, we assume all variables have been transformed into the local frame of an example robot $j\in \mathcal{R}$ located at $\q_j$, with $\q_j = [0~0]^{\top}$.

\subsection{Exact LoS-Distance Evaluation}
\label{sec_visible_region_and_metrics}

\begin{figure}[!t]
\vspace{6pt} 
\centering
\includegraphics[width=.9\linewidth]{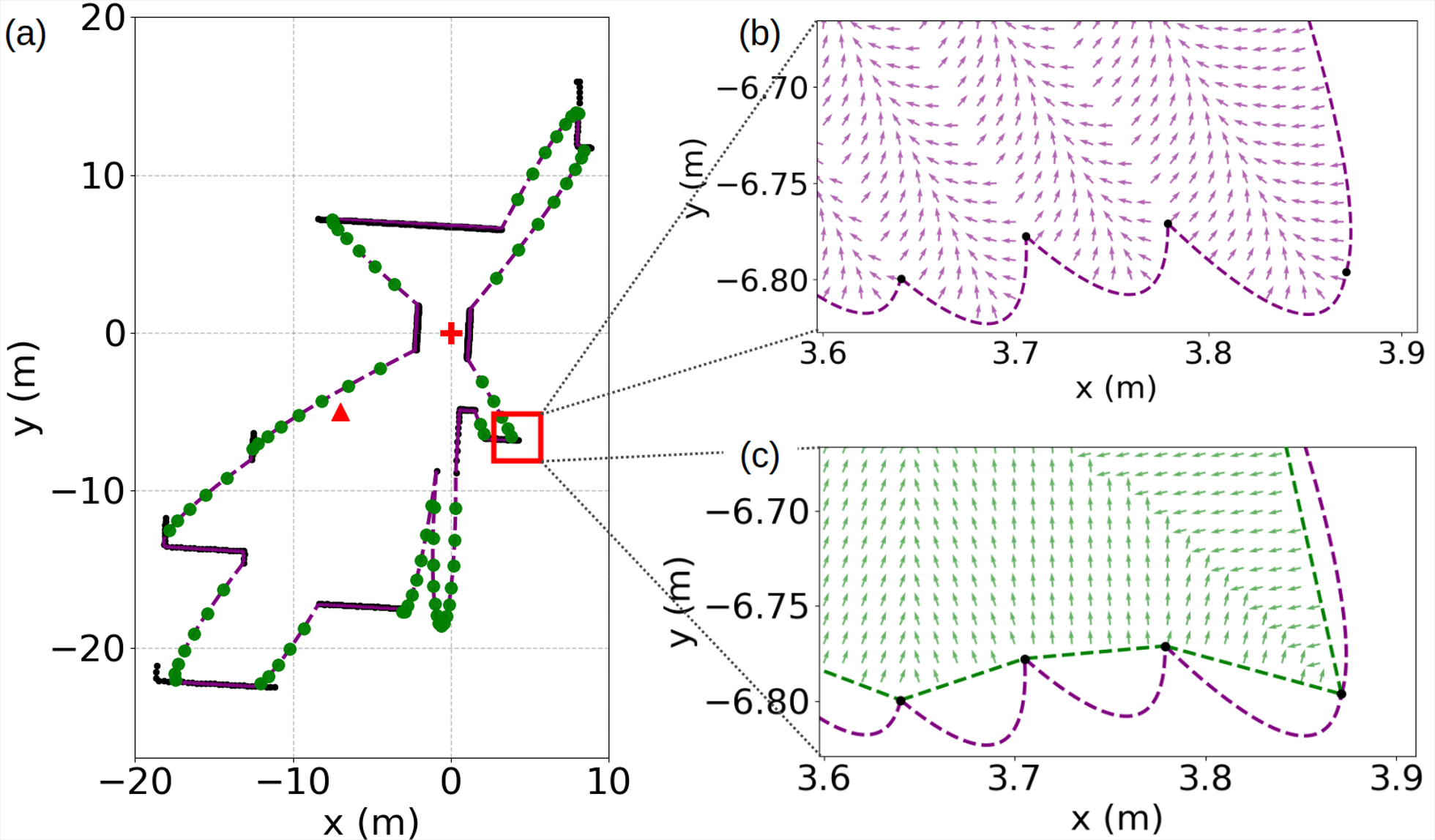}
\vspace{-8pt} 
\caption{(a): A robot (denoted by $+$) with its visible region and polygonal approximation, whose boundaries are highlighted by purple and green dotted lines in (b) and (c), respectively. 
% The green dots are the additionally interpolated points with $\Delta \theta = 0.5^{\circ}$.
The radial angle-based interpolation adapts to the boundary geometry by placing denser points (green dots in (a)) along boundaries of high curvature and sparser points along flatter boundaries.
% Polygonal approximation of a robot's (denoted by $+$) visible region (enclosed by purple dotted lines) from a frame of 2D LiDAR scan. 
(b) and (c) also show the gradient field of $d^{\text{hull}}$ (purple arrows) and $\tilde{d}^{\text{los}}$ (green arrows) respectively, in a zoomed-in area. 
A probe point (denoted by $\blacktriangle$) in (a) has the exact $d^{\text{los}} = 1.08m$ and $\tilde{d}^{\text{los}} = 1.07m$, while $d^{\text{hull}} = 12.50m$ showing significant metric error.
Moreover, the gradient of $d^{\text{hull}}$ tends to trap a robot within a small radial sector.}
\label{fig_linear_approx}
\vspace{-10pt}
\end{figure}

To efficiently quantify the distance to the visible region, previous works directly use the determination metric $d^{\text{hull}}$ in Eq.~(\ref{eq_dx}) or its variations as metrics~\cite{liu_StarconvexConstrained_2022, xia_RELINKRealTime_2023, bai2025realmrealtimelineofsightmaintenance}.
% Specifically, given a LiDAR scan $\mathcal{C}_{j}$ of a robot $j$ and a testing point $\q$, if the metric $d^{\text{hull}}$ of $f(\q)$ is large (small), it indicates that $\q$ is away from (close to) the boundary of $\mathcal{W}_j$.
% However, we find that $d^{\text{hull}}$ cannot provide accurate distance evaluation, and the gradient derived from $d^{\text{hull}}$ can lead to confined movement of robots.
% In this paper, we define the line-of-sight distance based on the visible region $\mathcal{W}_{j}$ as follows:
However, in this paper, we define the line-of-sight distance based on the visible region $\mathcal{W}_{j}$ as follows:

\begin{definition}[Line-of-Sight-Distance]
Given a robot $i$ located at $\q_{i}\in \mathcal{W}_j$, the LoS-distance of a robot $i$ to the visible region $\mathcal{W}_j$ of a robot $j$ is defined as
\[d^{\text{los}}_{ji} = \min\{\Vert \q_{i} - \p \Vert_{2} ~\vert~ \p\in \partial\mathcal{W}_{j}\},\]
where $\partial\mathcal{W}_{j}$ denotes the boundary of $\mathcal{W}_{j}$.
\label{def_los_distance}
\end{definition}
% \begin{remark}
%     The determination metric $d^{\text{hull}}$ captures the distance to a convex hull $\mathcal{H}_{j}$ (\ie, the projection of the actual visible region in the flipped space), while Def.~\ref{def_los_distance} directly evaluate the distance to the boundary visible region $\mathcal{W}_j$ in the original space.
% \end{remark}
% \begin{figure}[!t]
% \centering
% \includegraphics[width=.5\linewidth]{image/fig_show_sensitivity2.png}
% \caption{The illustration of the sensitivity of losing LoS in 2D case.}
% \label{fig_sensitivity2}
% \end{figure}

We abandon the previous metric $d^{\text{hull}}$ or its derivatives based on our observation and analysis that, it only captures the \emph{radial} distance to the boundary of a visible region while omitting the tangent distance, which leads to significant evaluation error.
Moreover, the gradient of $d^{\text{hull}}$ constrains the robot within a narrow sector with a small radial angle.
Consequently, a robot with only a small lateral clearance to the boundary can neither recognize such a situation nor escape from the critical region, as illustrated in Fig.~\ref{fig_linear_approx}.

% Consequently, a robot near the boundary can neither recognize such a situation nor escape from the critical region, as illustrated in Fig.~\ref{fig_linear_approx}.

In contrast, the new metric $d^{\text{los}}_{ji}$ evaluates the exact minimum distance from a robot $i$ to the boundary of robot $j$' visible region.
However, solving the accurate $d^{\text{los}}_{ji}$ as in Def.~\ref{def_los_distance} is a non-linear optimization problem that does not have an analytical solution (details are provided in the Appendix-A).
% Moreover, its high time complexity prevents it from real-time applications, as compared in Tab.~\ref{tab_approx_los_dist}.
Therefore, we resort to an efficient and safe approximation of the exact $d^{\text{los}}_{ji}$ that has a differentiable expression and supports real-time evaluation, introduced in the next subsection.

\subsection{Visible Region Approximation}

To support efficient LoS-distance evaluation, we use a 2D polygon (or polyhedron in 3D) to approximate the irregular boundaries of visible regions, as shown in Fig.~\ref{fig_framework} and Fig.~\ref{fig_linear_approx}.
Specifically, 
given the flipped convex hull $\mathcal{H}_{j}$ of a robot $j$, we first apply radial angle-based interpolation for each edge of $\mathcal{H}_{j}$, where the $k$-th edge will be interpolated with additional vertices if the radial angle $\theta_{k}$ spanned by this edge w.r.t. $\q_j$ exceeds a predefined threshold $\Delta \theta$.
% \footnote{The motivation behind the radial angle-based interpolation is that, if a longer edge of $\mathcal{H}_j$ covers a small radial angle, it is almost parallel to a robot's line-of-sight and should be well-approximated with a line segment. 
% We found that such interpolation adapts to the boundary geometry by placing denser points along boundaries of high curvature and sparser points along flatter boundaries, as shown in Fig.~\ref{fig_linear_approx}(a).}.
% interpolate each edge of the convex hull, \eg, the $k$-th edge, with 
After this, any two consecutive vertices on the boundary of $\mathcal{H}_j$ will span a radial angle not exceeding $\Delta \theta$.
Then, all vertices of $\mathcal{H}_{j}$, denoted as $\mathcal{V}_{\text{hull}}$, are flipped back to the original space following the flipping operation $f(\cdot)$. 
These vertices are sequentially connected with line segments to finally enclose an \emph{approximated visible region} $\tilde{\mathcal{W}}_j$ to approximate $\mathcal{W}_j$.
The detailed process is summarized in Alg.~\ref{alg_approx_visible}.

\begin{proposition}
    The time complexity of constructing the approximated visible region $\tilde{\mathcal{W}}_j$ is $\mathcal{O}(nlog(n)+m)$, where $n=\vert \mathcal{C}_j\vert$ and $m = \frac{2\pi}{\Delta \theta}$.
\end{proposition}

The approximated visible region $\tilde{\mathcal{W}}_j$ has several advantages. 
First, the analytical expression of the LoS-distance to $\tilde{\mathcal{W}}_j$ can be derived.
Specifically, the distance of a robot $i$ located at $\q_{i}$ to the $k$-th line segment $\overline{\boldsymbol{a}_{k}\boldsymbol{a}_{k+1}}$ of $\tilde{\mathcal{W}}_j$ can be calculated as
\begin{equation}
    \tilde{d}_{k} = 
    \begin{cases} 
        \Vert \q_{i}\boldsymbol{a}_{k} \Vert, & \frac{(\q_{i}\boldsymbol{a}_{k})^{\top}(\boldsymbol{a}_{k+1}\boldsymbol{a}_{k})}{\Vert \boldsymbol{a}_{k+1}\boldsymbol{a}_{k}\Vert} < 0, \\
        \Vert \q_{i}\boldsymbol{a}_{k+1} \Vert, & \frac{(\q_{i}\boldsymbol{a}_{k})^{\top}(\boldsymbol{a}_{k+1}\boldsymbol{a}_{k})}{\Vert \boldsymbol{a}_{k+1}\boldsymbol{a}_{k}\Vert} > 1, \\
        \frac{\Vert \q_{i}\boldsymbol{a}_{k+1} \times \q_{i}\boldsymbol{a}_{k+1} \Vert}{\Vert \boldsymbol{a}_{k}\boldsymbol{a}_{k+1} \Vert}, & \text{otherwise}.
    \end{cases}
\label{def_dk}
\end{equation}
The derivative of $\tilde{d}_{k}$ w.r.t $\q_i$ is provided in Appendix-B.
We define the \emph{approximated} LoS-distance  $\tilde{d}^{\text{los}}_{ji}$ as the distance from robot $i$ to the closest boundary of $\tilde{\mathcal{W}}_j$, calculated as
\begin{equation}
    \tilde{d}^{\text{los}}_{ji} = \min_{k} \tilde{d}_{k}.
\end{equation}

Since the mutual LoS-distance between two robot $i$ and $j$ can be different (or imbalanced) as mentioned in~\cite{bai2025realmrealtimelineofsightmaintenance}, \ie, $\tilde{d}^{\text{los}}_{ji}\ne \tilde{d}^{\text{los}}_{ij}$, we take the smaller value between them as the final LoS-distance to ensure conservative evaluation, defined as $D_{ij} = \min\{\tilde{d}^{\text{los}}_{ji}, \tilde{d}^{\text{los}}_{ij}\}$.
The potential function $\beta(\cdot)$ for LoS-connectivity constraints (C2) in Sec.~\ref{sec_types_of_constraints} is then defined as~\cite{bai2025realmrealtimelineofsightmaintenance}
\begin{equation}  
\small
\beta_{ij} =\left \{
    \begin{aligned}
        &0, &0\le D_{ij} < d^{\text{los}}_{\text{min}},\\
        &\frac{k_{\beta}}{2}[1-\cos(\frac{D_{ij} - d^{\text{los}}_{\text{min}}}{d^{\text{los}}_{\text{max}} - d^{\text{los}}_{\text{min}}})\pi], &d^{\text{los}}_{\text{min}}\le D_{ij} < d^{\text{los}}_{\text{max}},\\
        &k_{\beta}, &D_{ij} \ge d^{\text{los}}_{\text{max}},
    \end{aligned}
    \right.
\label{eq_def_beta}
\end{equation}
where $d^{\text{los}}_{\text{max}} >0$ is the trigger distance at which the LoS constraints take effect until robots lose LoS at $D_{ij} < d^{\text{los}}_{\text{min}}$, with $d^{\text{los}}_{\text{min}}$ being a small positive margin; $k_{\beta}>0$ is a scalar weight set to 1. 
It holds that $\beta_{ij} = \beta_{ji} = \beta(D_{ij})$.

The gradient direction for robot $i$ to enhance LoS-connectivity (\ie, increase $\beta_{ij}$) is given by
\begin{equation}
\small
    \frac{\partial \beta_{ij}}{\partial \q_i} = 
    \frac{\partial \beta(\tilde{d}^{\text{los}}_{ji})}{\partial \tilde{d}^{\text{los}}_{ji}} \vrule_{\tilde{d}^{\text{los}}_{ji} = D_{ij}} 
    \cdot
    \frac{\partial \tilde{d}^{\text{los}}_{ji}}{\partial \q_i}
\label{eq_gradient_visible_region}
\end{equation}
However, this gradient only pushes robot $i$ away from the boundary of $\tilde{\mathcal{W}}_{j}$.
When robot $i$ is free (\ie, no external target) and robot $j$ is at risk of losing LoS, we further expect robot $i$ to \emph{actively} move toward robot $j$ as a relay robot for connectivity maintenance.
Therefore, we reshape the gradient field by adding an additional radial component directed toward neighboring robots as follows:
\begin{equation}
\small
    % \frac{\partial d}{\partial \q_{i}} = \frac{\partial \tilde{d}^{\text{los}}_{ji}}{\partial \q_{i}} + \beta_{ji} \frac{-\q_{i}}{\Vert \q_{i}\Vert}.
    \frac{\partial \beta_{ij}}{\partial \q_i} \xleftarrow{} 
    \frac{\partial \beta(\tilde{d}^{\text{los}}_{ji})}{\partial \tilde{d}^{\text{los}}_{ji}} \vrule_{\tilde{d}^{\text{los}}_{ji} = D_{ij}} 
    \cdot
    \left(\frac{\partial \tilde{d}^{\text{los}}_{ji}}{\partial \q_i}
    +
     \beta(\tilde{d}^{\text{los}}_{ji}) \cdot \frac{-\q_{i}}{\Vert \q_{i}\Vert}
    \right).
\label{eq_local_gradient}
\end{equation}
When $\beta(\tilde{d}^{\text{los}}_{ji}) < k_{\beta}$, robot $i$ is mainly motivated to move away from the boundary of $\tilde{\mathcal{W}}_j$ following Eq.~(\ref{eq_gradient_visible_region}); and when $\beta(\tilde{d}^{\text{los}}_{ji}) = k_{\beta}$, robot $i$ will be further directed to chase robot $j$ if $D_{ij}$ is critical.
Note that the gradient in Eq.~(\ref{eq_local_gradient}) is defined in robot $j$'s local frame, and should be transformed into a global frame before being applied to robot $i$.

% The gradient of the approximated LoS-distance can push a robot be away from the closest boundary of the neighboring robot's visible region, thereby maintaining LoS-connectivity.

Second, we have the following propositions that guarantee $\tilde{\mathcal{W}}_j$ is a \emph{safe} approximation of $\mathcal{W}_j$, in the sense that the LoS-distance will never be overestimated.

\SetAlgoSkip{medskip}
\begin{algorithm}[t]
\label{alg_approx_visible}
\small
\SetKwInOut{Input}{Input}\SetKwInOut{Output}{Output}
\SetKwInOut{Return}{Return}
% \SetAlgoLined 控制算法控制语句是否以end结尾
\caption{ApproxVisibleRegion($\mathcal{H}_{j}$, $\Delta \theta$)}
$\mathcal{V}_{\text{hull}} \xleftarrow{} [~]$.\\
\For{$k\in [1, 2, ..., K]$}{
    $\langle \boldsymbol{a}_{k}, \boldsymbol{a}_{k+1} \rangle \xleftarrow{}$ the $k$-th edge of $\mathcal{H}_j$.\\
    Add $\boldsymbol{a}_{k}$ to $\mathcal{V}_{\text{hull}}$.\\
    $\theta_{k} = \cos^{-1}\frac{\boldsymbol{a}_{k}^{\top} \boldsymbol{a}_{k+1}}{\Vert\boldsymbol{a}_{k} \Vert  \Vert \boldsymbol{a}_{k+1}\Vert}$.\\
    $n_{\text{inter}} = \lfloor\theta_{k}/\Delta\theta\rfloor -1$.\\
    \If{$n_{\text{inter}} > 0$}{
        Intepolate edge $\langle \boldsymbol{a}_{k}, \boldsymbol{a}_{k+1} \rangle$ with $n_{\text{inter}}$ points.\\
        Sequentially add $n_{\text{inter}}$ points to $\mathcal{V}_{\text{hull}}$.
    }   
    Add $\boldsymbol{a}_{k+1}$ to $\mathcal{V}_{\text{hull}}$.\\
}
% $V_{\text{vis}} = f(\mathcal{V}_{\text{hull}})$.\\
% Spherically flip all points in $\mathcal{V}_{\text{hull}}$ back to original space to get $V_{\text{vis}}$.\\
$\tilde{\mathcal{W}}_j\xleftarrow{}$ the region enclosed by connecting consecutive vertices in $f(\mathcal{V}_{\text{hull}})$ with line segments.\\
\KwRet{$\tilde{\mathcal{W}}_j$}\\
\end{algorithm}
% \vspace{-10pt}

\begin{proposition}
The approximated visible region $\tilde{\mathcal{W}}_j$ for a robot $j\in \mathcal{R}$ is a subset of its actual visible region $\mathcal{W}_j$, \ie, $\tilde{\mathcal{W}}_j\subseteq \mathcal{W}_j$.
\label{prop_subset}
\end{proposition}

\begin{proposition}
    The approximated LoS-distance $\tilde{d}^{\text{los}}$ is the lower bound of the exact LoS-distance $d^{\text{los}}$, \ie, $\tilde{d}^{\text{los}} \le d^{\text{los}}$.
\label{prop_lower_bound}
\end{proposition}

The detailed proofs of Prop.~\ref{prop_subset} and Prop.~\ref{prop_lower_bound} are provided in the Appendix C and D respectively due to space limitations.

\subsection{Topology Optimization through Masked Graph Laplacian}

In a graph Laplacian-based connectivity controller, the topology among robots is implicitly controlled by the preferred fielder eigenvalue $\lambda_2^{\text{min}}$, \ie, a larger $\lambda_2^{\text{min}}$ will result in denser connections between robots~\cite{robuffogiordano_PassivitybasedDecentralized_2013, bai2025realmrealtimelineofsightmaintenance}. 
However, the explicit topology specification may impose redundant connections between robots, which hinders their navigation efficiency, and sometimes leads to occasional livelocks around obstacles, as shown in Fig.~\ref{fig_snapshot_connectivity}.
% where robots are stuck around an obstacle because of redundant connections.
% There is no priority assigned to these edges, and robots' movements can simultaneously weaken multiple edges so that a large connectivity force is incured to keep them stucked.

To address the above limitations, we propose the topology optimization above the real-time connectivity graph, where only a minimal topology $\mathcal{T}(t)\subseteq \mathcal{G}(t)$ is selected to maintain based on the following criteria:
(1) $\mathcal{T}(t)$ contains no redundant edges to ensure connectivity;
(2) the edges in $\mathcal{T}(t)$ require minimal efforts to maintain considering the robots’ navigation movement.
% \begin{itemize}
%     \item 
%     \item The edges in $\mathcal{T}(t)$ require minimal efforts to maintain under the robots’ navigation movement.
% \end{itemize}
While quantifying these efforts is non-trivial, we circumvent explicit evaluation by directly examining the conflict between connectivity maintenance and navigational movement.
% While it is non-trivial to quantify such efforts, we propose a straightforward approach to evaluate the violation between connectivity maintenance and navigational movement.
Concretely, recall that the real-time edge weight in $\mathcal{G}$ is defined as $A_{ij} = \alpha_{ij}\cdot \beta_{ij} \cdot \gamma_{ij}$, a larger $A_{ij}$ indicates that two robots $i$ and $j$ are well connected under the disturbance of their navigational movements.
Otherwise, $A_{ij}$ would be small.

Therefore, we formulate the topology optimization problem as an MST search problem over $\mathcal{G}(t)$, where each edge weight is modified as 
\begin{equation}
\small
    A_{ij}^{\text{mst}} = -\alpha_{ij}\cdot \beta_{ij} + \frac{d_{ij}}{d_{\text{max}}^{\text{com}}}.
\label{eq_mst_weight}
\end{equation}
The first term in Eq.~(\ref{eq_mst_weight}) indicates the efforts required to maintain the connectivity constraints. 
Note here we omit the weight $\gamma_{ij}$ for collision avoidance (C3) because it is always applied between robots to ensure safety.
Moreover, the second term, $\frac{d_{ij}}{d^{\text{com}}_{\text{max}}}$, serves to prioritize shorter edges.
The process of topology optimization is illustrated in Fig.~\ref{fig_framework}.
We use the Kruskal's algorithm to find the optimal spanning tree $\mathcal{T}$~\cite{kruskal1956shortest}.
% The process can also be distributed by utilizing distributed MST search algorithms, like the GHS algorithm~\cite{gallager1983distributed, pandurangan2018distributed}.

After obtaining the minimal topology $\mathcal{T}$, we mask the elements in the original generalized graph Laplacian matrix $\Lp$ as follows:
\begin{itemize}
    \item if $(i, j) \notin \mathcal{T}$, $A_{ij} = 0$ if $\gamma_{ij} = 1$, else $A_{ij} = \gamma_{ij}$.
    \item if $(i, j)\in \mathcal{T}$, $A_{ij}$ remains unchanged;
\end{itemize}
Here $\gamma_{ij} = 1$ indicates the collision avoidance constraints (C3) are not triggered.
The diagonal elements in $\Lp$ are then modified accordingly.
Note that for those edges not in $\mathcal{T}$, we only keep $\gamma_{ij}$ to ensure safety, while eliminating their communication and LoS maintenance constraints.
Based on the masked $\Lp$, the connectivity velocity can finally be derived following Eq.~(\ref{eq_connectivity_force}) as
\begin{equation}
\small
      \boldsymbol{u}^{\text{c}}_i =  -\frac{1}{(\lambda_2 - \lambda_{2}^{\text{min})^2}}\cdot \sum_{j\in \mathcal{N}_{i}} \frac{\partial A_{ij}}{\partial \q_{i}} \cdot (v_{2_i} - v_{2_j})^{2},
\label{eq_connect_force_final}
\end{equation}
where $\frac{\partial A_{ij}}{\partial \q_{i}} = \frac{\partial\alpha_{ij}}{\partial \q_{i}}\cdot \gamma_{ij}\cdot \beta_{ij}  + \alpha_{ij}\cdot \frac{\partial \gamma_{ij}}{\partial \q_{i}}\cdot \beta_{ij} + \alpha_{ij}\cdot \gamma_{ij}\cdot \frac{\partial \beta_{ij}}{\partial \q_i}$; $\lambda_{2}$ and $\boldsymbol{v}_{2}$ are the Fiedler eigenvalue and eigenvector of $\Lp$.

\begin{remark}
The graph Laplacian-based controller may fail to maintain a specific edge in $\mathcal{G}$, because $\lambda_2$ is only non-decreasing (rather than strictly increasing) w.r.t. the number of edges in $\mathcal{G}$~\cite{fiedler1973algebraic}, \ie, preserving a specific edge may not contribute to $\lambda_2$. 
However, in our case, since we reduce $\mathcal{G}$ to an MST by topology optimization, the $\lambda_2$ of the masked $\Lp$ is now strictly increasing w.r.t. the number of remaining edges.
Therefore, correct $\boldsymbol{u}_{i}^{\text{c}}$ can be derived to maintain edges in $\mathcal{T}$.
\end{remark}

\begin{remark}
When topology optimization is performed using distributed MST algorithms like GHS algorithm~\cite{gallager1983distributed}, the connectivity velocity in Eq.~(\ref{eq_connectivity_force}) can be calculated distributively by each robot, requiring only one-hop communication with its neighbors. The proof is similar to the proof in~\cite{bai2025realmrealtimelineofsightmaintenance}.
\end{remark}

\section{Experiment Results}
\label{sec_experiment}

\begin{table}[!t]
\vspace{6pt} 
\caption{Approximation error of LoS-distance}
\vspace{-8pt}
\scriptsize
\centering
\begin{tabular}{c|c|c|c|r|c|r}
\Xhline{0.8pt}
$r_{\text{flip}}$ & Metrics                                   & $\Delta\theta$ & $\vert \mathcal{V}_{\text{hull}}  \vert$              & $t_{\text{avg}}\downarrow$ & $e_{\text{avg}}\downarrow$ & $e_{\text{max}}\downarrow$  \\ \Xhline{0.6pt}
\multirow{4}{*}{150}  & $d^{\text{los}}$                          & --                       & 428                                        & 1157.84               & 0.00                      & 0.00                      \\ \cline{2-7} 
                      & \multirow{3}{*}{$\tilde{d}^{\text{los}}$} & --                       & 428                                        & 0.16                  & 53.25                  & 253.94                 \\  
                      &                                           & 2                         & 471\scriptsize{($+10\%$)}               & 0.16                  & 1.49                   & 4.58                   \\  
                      &                                           & 1                         & 526\scriptsize{($+23\%$)}   & 0.17                  & \textbf{0.34}                   & \textbf{1.20}                   \\ \hline
\multirow{4}{*}{500}  & $d^{\text{los}}$                          & --                        & 577                                        & 2243.31               & 0.00                      & 0.00                      \\ \cline{2-7} 
                      & \multirow{3}{*}{$\tilde{d}^{\text{los}}$} & --                        & 577                                        & 0.17                  & 19.76                  & 119.32                 \\  
                      &                                           & 2                         & 601\scriptsize{($+4\%$)} & 0.17                  & 2.82                   & 11.65                  \\  
                      &                                           & 1                         & 640\scriptsize{($+11\%$)} & 0.18                  & \textbf{0.60}                   & \textbf{1.88}                   \\ \hline
\multirow{4}{*}{1000} & $d^{\text{los}}$                          & --                        & 608                                        & 3416.79               & 0.00                      & 0.00                      \\ \cline{2-7} 
                      & \multirow{3}{*}{$\tilde{d}^{\text{los}}$} & --                        & 608                                        & 0.17                  & 13.52                  & 78.63                  \\ 
                      &                                           & 2                         & 624\scriptsize{($+3\%$)} & 0.20                  & 4.36                   & 25.60                  \\ 
                      &                                           & 1                         & 655\scriptsize{($+8\%$)} & 0.21                  & \textbf{0.88}                   & \textbf{4.24}                   \\ \Xhline{0.8pt}
\end{tabular}\\
\label{tab_approx_los_dist}
\vspace{1mm}
\begin{minipage}{\linewidth}
\scriptsize Notes: $r_{\text{flip}}$ is the flipping radius (m); $\Delta \theta$, the interpolation threshold ($^\circ$); $t_{\text{avg}}$, the averaged computational time (ms); $e_{\text{avg}}$ and $e_{\text{max}}$, the averaged and the maximum approximation errors (cm).
% The units of the terms are as follows: $r_{\text{flip}}$ (m), $\Delta \theta$ ($^\circ$), $t_{\text{avg}}$ (ms), $e_{\text{avg}}$ (cm), $e_{\text{max}}$ (cm).
"--" means no interpolation is performed.
% The experiments are conducted with a frame of 2D LiDAR scan with 718 points, as shown in Fig.~\ref{fig_linear_approx}(a).
The percentages show the number of interpolated points w.r.t. the number of the original hull vertices.
\end{minipage}
\vspace{0pt}
\end{table}

\begin{figure}[!t]
\centering
\includegraphics[width=\linewidth]{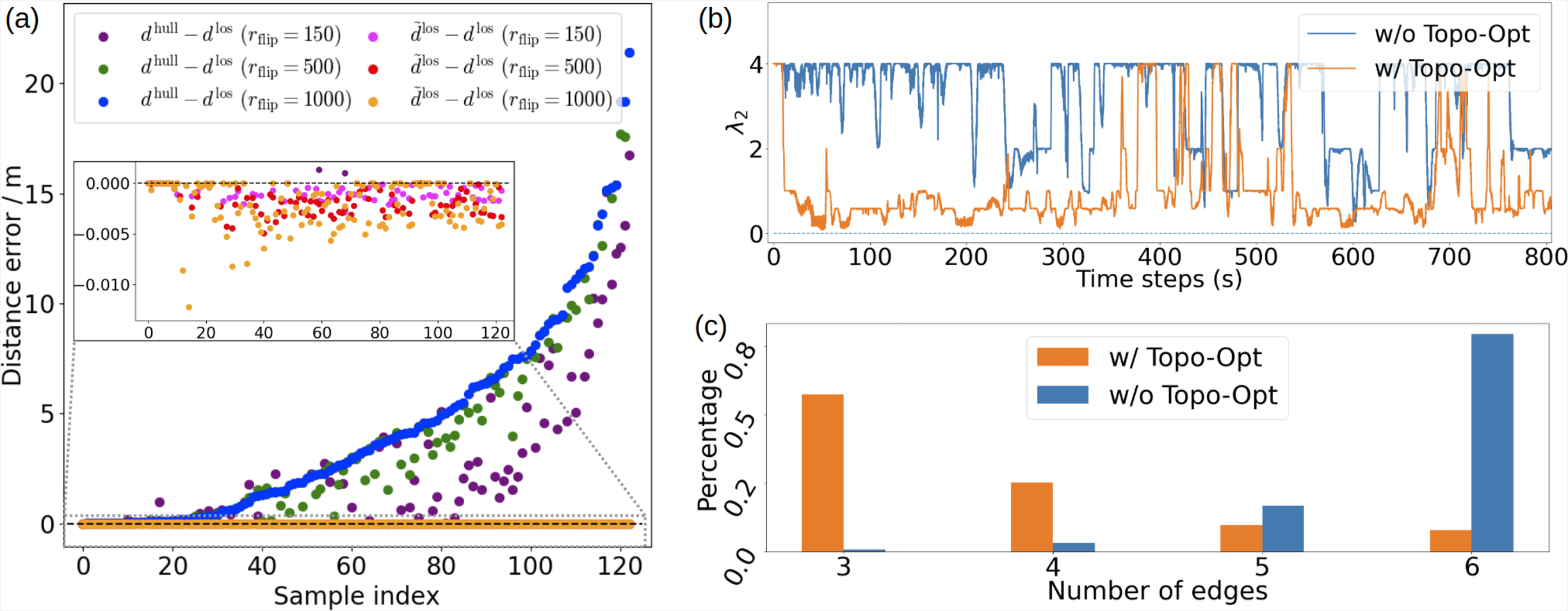}
\vspace{-18pt} 
\caption{(a) Error distribution of different LoS-distance metrics at uniformly sampled points in a visible region shown in Fig.~\ref{fig_linear_approx}, under various flipping radius $r_{\text{flip}}$. Here we set $\Delta\theta = 0.5^{\circ}$. 
(b) and (c): Ablation study of using Topo-Opt when four robots are exploring a garage environment in Fig.~\ref{fig_wide_grid}(b), where (b) compares the temporal variation of $\lambda_2$ and (c) shows the histogram of the number of connections between four robots during exploration.}
\label{fig_los_dist_compare}
\vspace{-10pt}
\end{figure}

This section evaluates the proposed framework in terms of the effectiveness of LoS-distance approximation, the robustness of connectivity maintenance, and navigation efficiency in multi-robot navigation and exploration tasks.
% This section evaluates the performance of LoS-distance approximation, the robustness of the proposed LoS-connectivity maintenance framework, and navigation efficiency with topology optimization in multi-robot navigation and exploration tasks.
The target assignment and the calculation of navigation velocity follow the design in~\cite{bai2025realmrealtimelineofsightmaintenance}, as shown in Fig.~\ref{fig_framework}. 
% We employ the framework developed in~\cite{bai2025realmrealtimelineofsightmaintenance} for task assignment, path planning, and navigation velocity calculation. 
We compare three different metrics and several topology strategies, including a fixed topology, a graph Laplacian-based topology~\cite{bai2025realmrealtimelineofsightmaintenance}, and our proposed adaptive minimal topology, and conduct a series of ablation studies.
% The proposed method is compared with a baseline method~\cite{bai2025realmrealtimelineofsightmaintenance}, a fixed topology-based method, and several ablation studies.
All simulation experiments are conducted in Gazebo on a desktop with an i9-13900 CPU and 32 GB of RAM, where robots are equipped with 2D LiDARs with 360 degrees of FoV and a sensing range of 30 meters.
The real-world experiments are introduced in Sec.~\ref{sec_real_world_experiment}.
The laser points that hit neighboring robots are removed before constructing the visible regions.
Our connectivity controller runs at $30$ Hz thanks to its inherently distributed calculation.

% For multi-robot exploration, we adopt Karto~\cite{Konolige_Karto_2010} for multi-robot mapping.

\subsection{Performance of LoS-Distance Approximation}
\label{sec_experiment_los_approximation}

This section compares the computational efficiency and approximation accuracy of the approximated LoS-distance $\tilde{d}^{\text{los}}$ versus the exact $d^{\text{los}}$.
The metrics are evaluated at uniformly sampled positions within the visible region in Fig.~\ref{fig_linear_approx}.
% The results are shown in Tab.~\ref{tab_approx_los_dist} and Fig.~\ref{fig_los_dist_compare}(a).
% The results are shown in Tab.~\ref{tab_approx_los_dist} and Fig.~\ref{fig_los_dist_compare}.
As shown in Tab.~\ref{tab_approx_los_dist}, the averaged computational time of $\tilde{d}^{\text{los}}$ is less than one millisecond thanks to its closed-form expression, while solving $d^{\text{los}}$ takes more than a second, which prevents it from real-time applications.
% In contrast, our method only takes less than one millisecond given the same Lidar scan, thanks to the closed-form expression of $\tilde{d}^{\text{los}}$.
Moreover, the proposed $\tilde{\mathcal{W}}_j$ well approximates the irregular boundary of the original visible region by applying radial angle-based interpolation, 
with average and maximum errors less than 5 cm and 30 cm, respectively, even under aggressive visible regions (\ie, larger $r_{\text{flip}}$) and coarse interpolation.

The detailed distribution of the approximation error is shown in Fig.~\ref{fig_los_dist_compare}(a), where we further include the previous metric $d^{\text{hull}}$.
Aligned with our analysis in Sec.~\ref{sec_visible_region_and_metrics}, the metric $\tilde{d}^{\text{los}}$ is consistent with $d^{\text{los}}$ with small approximation errors, while $d^{\text{hull}}$ exhibits significant errors in some areas (larger than 20 meters), indicating it is unreliable as a LoS-distance metric. 
It is worth noting that $\tilde{d}^{\text{los}}$ remains consistently non-positive, which experimentally validates our Prop.~\ref{prop_lower_bound}.
Generally, the approximation error increases with a larger flipping radius $r_{\text{flip}}$ and a coarser interpolation, which requires a trade-off between the accuracy and computation efficiency.

\subsection{Robustness with Different LoS-Distance Metrics}

This section evaluates the robustness of the proposed connectivity maintenance framework using different LoS-distance metrics, including $d^{\text{hull}}$~\cite{xia_RELINKRealTime_2023, liu_StarconvexConstrained_2022}, $d^{\text{hull}}\cos{\theta_{k^*}}$~\cite{bai2025realmrealtimelineofsightmaintenance}, and $\tilde{d}^{\text{los}}$ (proposed).
We conduct ablation studies to evaluate performance with and without topology optimization (Topo-Opt), under different trigger distance $d^{\text{los}}_{\text{max}}$ of LoS constraints, and for varying ranges of visible regions.
We deploy four robots to navigate to their respective targets in a environment cluttered with small and irregular obstacles that can frequently block the LoS between robots, as shown in Fig.~\ref{fig_wide_grid}(a).
% It is challenging for connectivity maintenance because small obstacles frequently block the LoS between robots.
% The performance is evaluated considering the existence of redundant connections (with or without Topo-Opt), different trigger distance for LoS maintenance (determined by $d^{\text{los}}_{\text{max}}$ in Eq.~(\ref{eq_def_beta})), and various aggressiveness of visible regions (determined by $r_{\text{flip}}$).

% Their respective targets are randomly generated within the environment.
% Note that it is challenging for robots to always maintain connectivity in such a environment  when they move to their respective targets.

The results for LoS-connectivity maintenance are summarized in Tab.~\ref{tab_success_rate_with_different_metrics}.
% with or without applying topology optimization, and under different flipping radius.
% Four robots are deployed in a multi-target navigation task within an unknown cluttered environment shown in Fig.~\ref{fig_wide_grid}, based solely on their onboard LiDAR sensors.
% For the metric $d^{\text{hull}}$, we also tried with a larger $d^{\text{los}}_{\text{max}}$ to reserve enough safety distance, as in previous works~\cite{bai2025realmrealtimelineofsightmaintenance, xia_RELINKRealTime_2023}.
% As we analyzed in Sec.~\ref{sec_visible_region_and_metrics}, it is challenging for $d^{\text{hull}}$ to reflect the actual LoS-distance with a larger flipping radius $r_{\text{flip}}$. 
% Additionally, accurate LoS-distance evaluation is even more critical when applying topology optimization, as there can be no redundant edges, and fail to maintain any connected edge can lead to the disconnection of the whole team.
% The results are shown in Tab.~\ref{tab_success_rate_with_different_metrics}, where we compare that whether the system can maintain team connectivity with or without applying topology optimization, and under different flipping radius.
% Four robots are deployed in a multi-target navigation task within an unknown cluttered environment shown in Fig.~\ref{fig_wide_grid}, based solely on their onboard LiDAR sensors.
% For the metric $d^{\text{hull}}$, we also tried with a larger $d^{\text{los}}_{\text{max}}$ to reserve enough safety distance, as in previous works~\cite{bai2025realmrealtimelineofsightmaintenance, xia_RELINKRealTime_2023}.
In general, $d^{\text{hull}}$ and $d^{\text{hull}}\cos{\theta_{k^*}}$ show similar performance: they perform well when Topo-Opt is disabled and a sufficiently large trigger distance $d^{\text{los}}_{\text{max}}$ is set.
However, they fail to ensure LoS connectivity under larger visible ranges when a smaller $d^{\text{los}}_{\text{max}}$ is set.
Moreover, when redundant connections are removed with Topo-Opt, both metrics consistently fail regardless of the trigger distance and flipping radius $r_{\text{flip}}$.
This is due to the inherent inaccuracy of $d^{\text{hull}}$ as we previously shown in Sec.~\ref{sec_experiment_los_approximation}.
In contrast, our metric $\tilde{d}^{\text{los}}$ can consistently preserve connectivity under both aggressive visible ranges (with $r_{\text{flip}} = 1000\,\text{m}$) and fragile connectivity topology (with almost no redundant connections), verifying its reliability for LoS-distance evaluation.
% Moreover, our method can support very aggressive formulation of visible regions when $r_{\text{flip}} = 1000\,\text{m}$. 
These advantages enable our framework to simultaneously cover a larger area while reliably maintaining minimal connections among robots.

% , verifying that previous metrics~\cite{bai2025realmrealtimelineofsightmaintenance, xia_RELINKRealTime_2023} cannot be used for reliable LoS-distance evaluation.

% even with an earlier trigger distance $d^{\text{los}}_{\text{max}}$, which is 

% When Topo-Opt is disabled and a sufficiently large trigger distance $d^{\text{los}}_{\text{max}}$ is set, we found that the metric $d^{\text{hull}}$ performs well because the redundant connections (explicitly determined by $\lambda_2$) help to secure the connectivity during navigation.
% However, due to inherent inaccuracy of $d^{\text{hull}}$, once a less conservative $d^{\text{los}}_{\text{max}}$ is set, \eg, $d^{\text{los}}_{\text{max}} = 1.2\,\text{m}$, the method with $d^{\text{hull}}$ fails to keep connectivity when $r_{\text{flip}}$ gets larger.
% Furthermore, when redundant connections are removed through Topo-Opt, the metric $d^{\text{hull}}$ fails to maintain connectivity in all trials, which indicates that simply setting a large trigger distance $d^{\text{los}}_{\text{max}}$ cannot ensure $d^{\text{hull}}$ a reliable metric for the LoS-distance evaluation as in previous works~\cite{bai2025realmrealtimelineofsightmaintenance, xia_RELINKRealTime_2023}.
% \red{Although $d^{\text{hull}}\cos{\theta_{k^*}}$ further considers sensiviity of LoS-distance, it still fails to overcome the limitation of $d^{\text{hull}}$, therefore presenting similar performance as $d^{\text{hull}}$.}

% In contrast, with the proposed metric $\tilde{d}^{\text{los}}$, 

To further showcase these advantages, we deploy the proposed framework in multi-robot exploration tasks, requiring robots to explore both structured and unstructured environments, as shown in Fig.~\ref{fig_wide_grid}(b) and (c).
Despite the existence of dense obstacles, the proposed method guarantees LoS-connectivity during the exploration process, verified by the positivity of $\lambda_2$ as shown in Fig.~\ref{fig_los_dist_compare}(b).
Moreover, our method maintains significantly fewer but necessary connections between robots when Topo-Opt is enabled, as evidenced by the histogram of the number of connections in Fig.~\ref{fig_los_dist_compare}(c).
The results demonstrate the reliability of our framework for connectivity maintenance even under fragile topologies with fewer connections.

\begin{table}[]
\vspace{6pt} 
\caption{Connectivity maintenance in multi-robot navigation}
\vspace{-8pt}
\scriptsize
\centering
\begin{tabular}{l|c|c|ccc}
\Xhline{0.8pt}  % 加粗顶部框线
% \hline
\multicolumn{1}{c|}{\multirow{2}{*}{Metrics}} & \multicolumn{1}{c|}{\multirow{2}{*}{Topo-Opt}} & \multicolumn{1}{c|}{\multirow{2}{*}{$d^{\text{los}}_{\text{max}}$(m) $\downarrow$}} & \multicolumn{3}{c}{$r_{\text{flip}}$(m) $\uparrow$}                                                  \\ \cline{4-6} 
\multicolumn{1}{c|}{}                         & \multicolumn{1}{c|}{}                          & \multicolumn{1}{c|}{}                                                   & \multicolumn{1}{c|}{150}            & \multicolumn{1}{c|}{500}            & 1000           \\ \Xhline{0.6pt}
\multirow{4}{*}{\begin{tabular}{@{}c@{}}$d^{\text{hull}}$~\cite{xia_RELINKRealTime_2023, liu_StarconvexConstrained_2022}; \\ [1pt]$d^{\text{hull}}\cdot\cos{\theta_k}$~\cite{bai2025realmrealtimelineofsightmaintenance}\end{tabular}} & w/o                                             & $3.0$                                                                   & \multicolumn{1}{c|}{\(\checkmark\)} & \multicolumn{1}{c|}{\(\checkmark\)} & \(\checkmark\) \\ 
                                              & w/o                                             & $1.2$                                                                   & \multicolumn{1}{c|}{\(\checkmark\)} & \multicolumn{1}{c|}{\(\times\)}     & \(\times\)     \\ \cline{2-6} 
                                              &w/                                            & $3.0$                                                                   & \multicolumn{1}{c|}{\(\times\)}     & \multicolumn{1}{c|}{\(\times\)}     & \(\times\)     \\ 
                                              & w/                                            & $1.2$                                                                   & \multicolumn{1}{c|}{\(\times\)}     & \multicolumn{1}{c|}{\(\times\)}     & \(\times\)     \\ \hline
\multirow{2}{*}{$\tilde{d}^{\text{los}}$ (Proposed)}            & w/o                                            & $1.2$                                                                   & \multicolumn{1}{c|}{\(\checkmark\)} & \multicolumn{1}{c|}{\(\checkmark\)} & \(\checkmark\) \\ 
                                              & w/                                            & $1.2$                                                                   & \multicolumn{1}{c|}{\(\checkmark\)} & \multicolumn{1}{c|}{\(\checkmark\)} & \(\checkmark\) \\ % \hline
                                              \Xhline{0.8pt}  % 加粗底部框线
\end{tabular}
\vspace{1mm}
\begin{minipage}{\linewidth}
\scriptsize Notes: $d^{\text{los}}_{\text{min}}$ is the trigger distance for LoS constraints; $d^{\text{los}}_{\text{min}}$ is set to $0.1\,\text{m}$ in all experiments. 
The mark \(\checkmark\) indicates robots successfully navigate to their targets while maintaining connectivity, and \(\times\) indicates failure. 
\end{minipage}
\label{tab_success_rate_with_different_metrics}
\vspace{-5pt}
\end{table}

\subsection{Navigation Efficiency and Applications}

% \begin{figure}[!t]
% \centering
% \includegraphics[width=\linewidth]{image/fig_four.png}
% \caption{Snapshots of real-world experiments, where three drones navigate in an environment with obstacles. More details are provided in the attached videos.}
% \label{fig_real_world}
% \end{figure}

% \begin{figure}[t]
% % \vspace{6pt}
% \centering
% \subfloat[Group (1): one target.]{\includegraphics[width=.495\linewidth]{image/fig_group1.png}}\hspace{-1pt}
% \subfloat[Group (2): four targets.]{\includegraphics[width=.49\linewidth]{image/fig_group2.png}}
% \caption{Comparison of time and distance efficiency in four-robot navigation in a cluttered environment shown in Fig.~\ref{fig_wide_grid}(a). Averaged results and standard deviations are obtained from five independent runs. Different cases correspond to different distributions of targets. (a) only one robot is assigned a random target in the environment; (b) all four robots have their targets. See the supplementary video for full robot trajectories and dynamic behaviors.}
% \label{fig_efficiency_one}
% \vspace{-15pt}
% \end{figure}

\begin{figure}[t]
% \vspace{6pt}
\centering
\includegraphics[width=\linewidth]{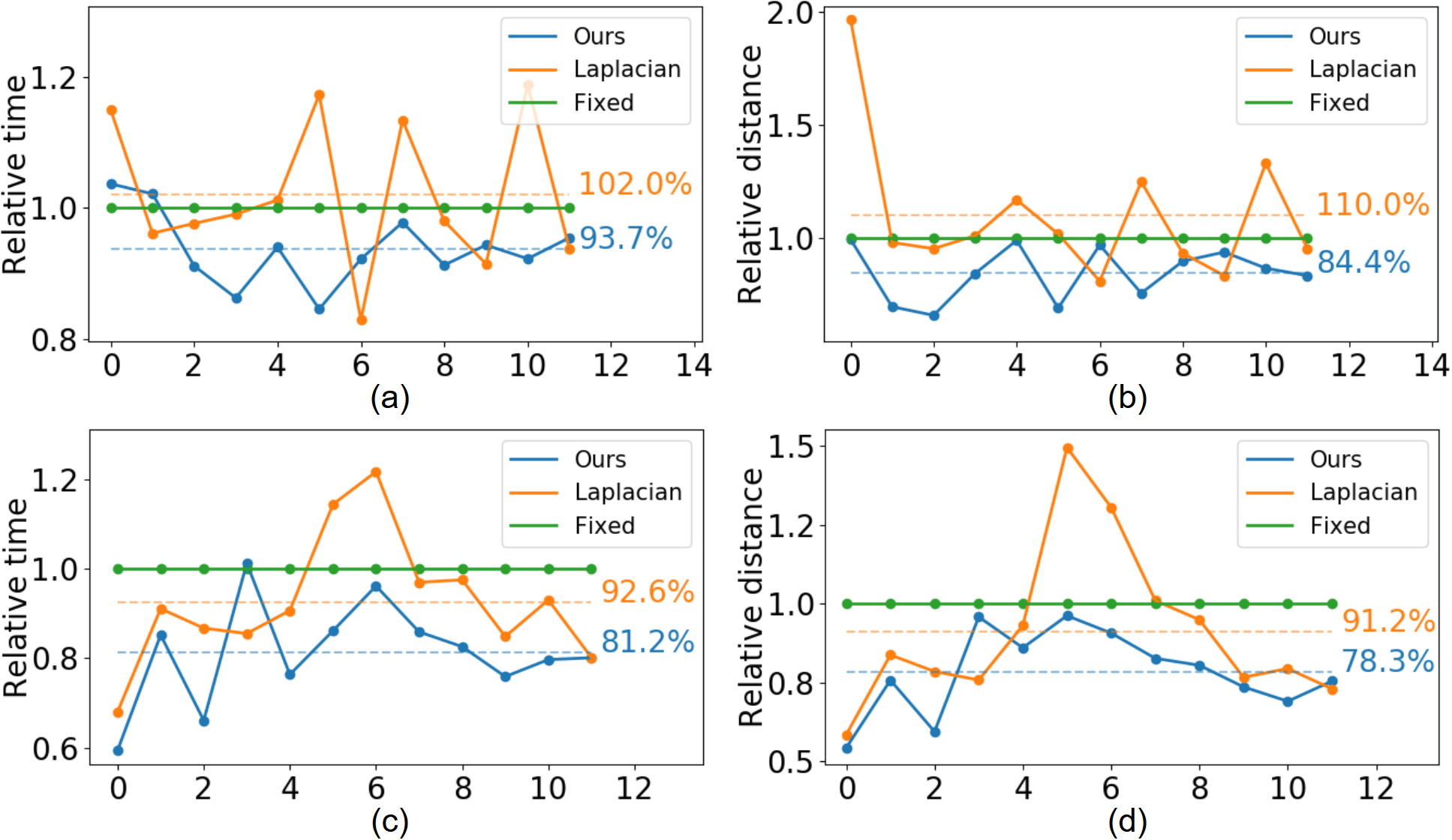}
\vspace{-20pt} 
\caption{Relative time and distance efficiency of compared methods with four robots navigating in a cluttered environment shown in Fig.~\ref{fig_wide_grid}(a). Dashed lines indicate the average performance.
The indices of different runs are sorted in ascending order based on the total distance to the targets.
(a) and (b): only one robot is assigned a random target; (c) and (d): all four robots have their targets. }
\label{fig_efficiency_one}
\vspace{-10pt}
\end{figure}

\begin{figure}[t]
\vspace{6pt}
\centering
\includegraphics[width=\linewidth]{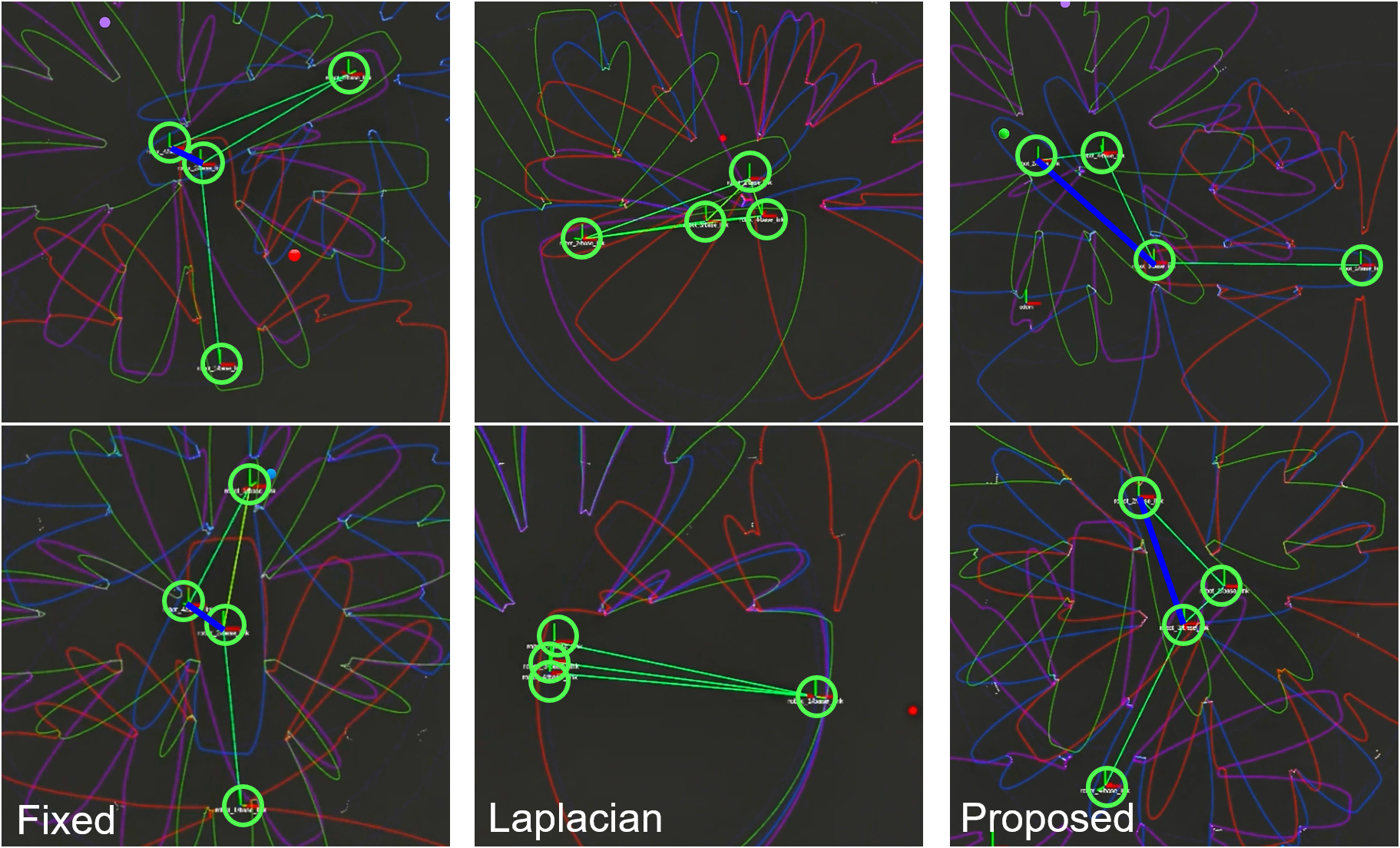}
\vspace{-18pt} 
\caption{Snapshots of connectivity graphs with the Fixed, Laplacian, and the proposed methods. The figure highlights the robots (green circles), selected connections (green edges), and redundant connections (blue edges). The Fixed method fails to adapt to a better topology to facilitate navigation efficiency; the Laplacian method tends to maintain redundant connections, and can lead to occasional livelock around obstacles; the proposed method adaptively optimizes the topology that facilitates navigation efficiency. The results are best appreciated with the video in supplementary materials.}
\label{fig_snapshot_connectivity}
\vspace{-10pt}
\end{figure}

% \begin{figure}[t]
% % \vspace{6pt}
% \centering
% \subfloat[Total time spent.]{\includegraphics[width=.48\linewidth]{image/fig_time1.png}}\hspace{1pt}
% \subfloat[Total distance travelled.]{\includegraphics[width=.48\linewidth]{image/fig_dist1.png}}\\
% \subfloat[Total time spent.]{\includegraphics[width=.47\linewidth]{image/fig_time2.png}}\hspace{1pt}
% \subfloat[Total distance travelled.]{\includegraphics[width=.49\linewidth]{image/fig_dist2.png}}
% \caption{Comparison of time and distance efficiency in multi-robot navigation tasks, with the averaged results and standard deviations in five independent runs. Different cases correspond to different distributions of navigation targets. In (a) and (b), only one robot is assigned a random target in the environment, while others only move to maintain connectivity; in (c) and (d), all four robots have their targets.}
% \label{fig_efficiency_one}
% \end{figure}
This section compares the navigation efficiency of related methods with different connectivity topologies, including a fixed topology (\textbf{Fixed}), the topology explicitly determined by $\lambda_2$ (\textbf{Laplacian}), and the topology obtained from Topo-Opt (\textbf{Ours}).
We conduct two groups of experiments with different proportions of free robots: in group (1), only one robot is assigned a target; while in group (2), all four robots have their respective targets. 
Each group includes twelve independent runs with randomly generated targets in the environment shown in Fig.~\ref{fig_wide_grid}(a), which contains dense and cluttered obstacles.

% with dense and cluttered obstacles. 

The relative time and distance efficiency w.r.t. the Fixed method are shown in Fig.~\ref{fig_efficiency_one}. The snapshots of connectivity graphs in group (2) is also depicted in Fig.~\ref{fig_snapshot_connectivity}.
In group (1) (Fig.~\ref{fig_efficiency_one}(a)(b)), the Laplacian method has the worst navigation efficiency because it maintains unnecessary connections that influence the movements of unrelated robots. 
In group (2), (Fig.~\ref{fig_efficiency_one}(c)(d)), the Fixed method shows the worst navigation efficiency as it blindly sticks to one fixed topology; the Laplacian method performs slightly better since the unnecessary connections help to alleviate conflicts in robots' navigational directions.
In both settings, the proposed method shows the best time and distance efficiency on average, with around $10\%$ and $20\%$ improvements w.r.t. the Laplacian and the Fixed methods, respectively, thanks to the topology optimization that minimizes the interference to navigation efficiency.
Moreover, the relative improvement of the distance efficiency increases from $15.6\%$ to $21.7\%$ as the number of free robots increases, indicating the potential of Topo-Opt in large scale multi-robot systems.
Finally, it is worth noting that all methods use the proposed LoS-distance metric $\tilde{d}^{\text{los}}$ to ensure robust connectivity.
This highlights the generality of our framework, which can support either the explicit minimal topology requirements (either adaptive or fixed) or the topologies implicitly determined by the Fiedler eigenvalue.

\subsection{Real-World Experiment}
\label{sec_real_world_experiment}

\begin{figure}[!t]
\vspace{6pt}
\centering
\includegraphics[width=\linewidth]{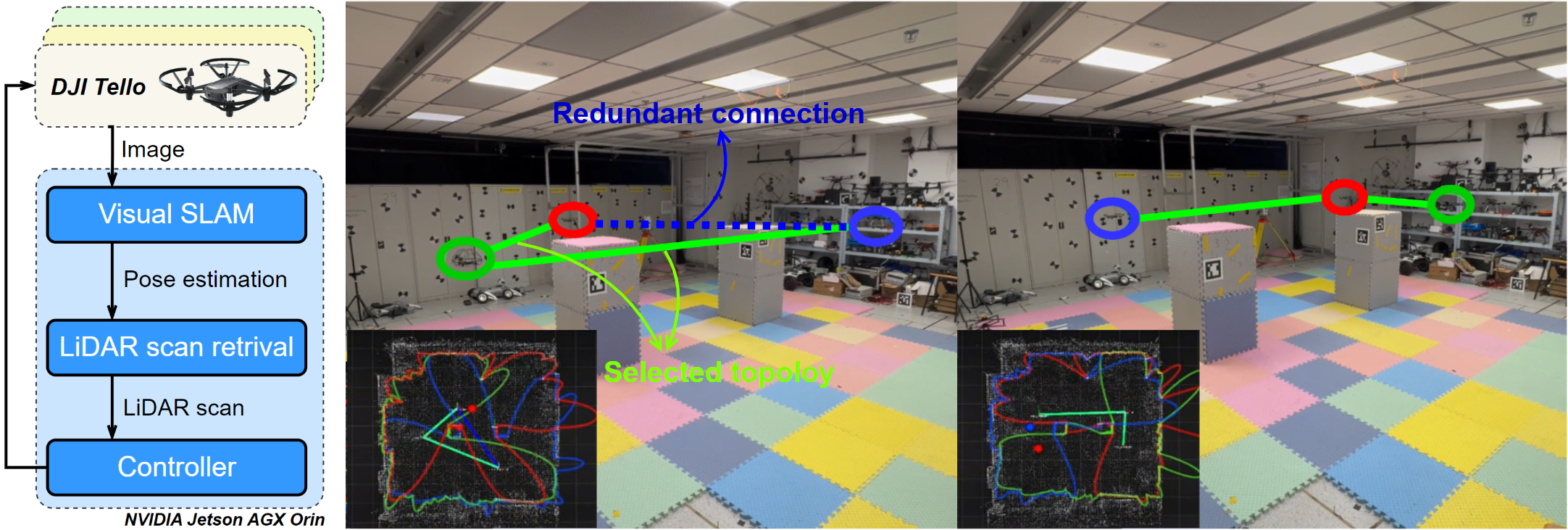}
\vspace{-20pt} 
\caption{Snapshots of real-world experiments where three drones navigate in an environment with obstacles. More details are in the attached videos.}
\label{fig_real_world}
\vspace{-10pt}
\end{figure}

We deploy the proposed method in real-world multi-robot navigation tasks with a multi-UAV system comprising three DJI Tello UAVs~\cite{li2025airswarmenablingcosteffectivemultiuav}, as shown in Fig.~\ref{fig_real_world}. 
Since the Tello UAV cannot mount LiDAR sensors, we pre-build a point cloud map of the environment and only use it to generate onboard LiDAR scans based on robots' real-time visual localization with onboard cameras~\cite{xu2025airslam}. 
% The localization of UAVs is obtained from an onboard visual SLAM system~\cite{xu2025airslam}.
While we manually controlled one drone to fly around obstacles, the other two drones can autonomously move to maintain LoS-connectivity and take adaptive topologies for navigation efficiency, verifying the effectiveness of the proposed method.
Notably, the entire connected multi-robot system can be deployed in unknown environments relying solely on onboard sensing and computation, underscoring its convenience and practicality for real-world deployment.

% Notably, the entire connected multi-robot system can be deployed in unknown environments based solely on onboard sensing and computational resources, which illustrates the advantages of our method for convenient and practical real-world deployment.

\section{CONCLUSIONS and Discussions}
This paper proposes a LoS-connectivity maintenance method that requires no prior environment maps, based solely on robots' real-time LiDAR scans.
We propose an efficient and reliable LoS-distance metric to support LoS constraints formulation, based on the polygonal approximation of robots' visible regions.
Unnecessary and effort-demanding connections are masked out to minimize the efforts for connectivity maintenance between robots
under external navigation tasks.
The robustness and efficiency of the proposed method are verified with extensive multi-robot navigation experiments in challenging environments with unknown and cluttered obstacles.
% which is proved to be a reliable lower-bound approximation of the actual LoS-distance.
% The connection topology is further optimized by masking edges that require significant efforts to maintain in traditional graph Laplacian-based controllers, which results in a more dispersed formation to enable flexible movements of robots.
% We verify the robustness of the system with extensive multi-robot navigation tasks in cluttered environments.
% Future work is to develop learning-based methods for unified optimization of connectivity maintenance and exploration coordination.
Future work will focus on developing learning-based methods for unified coordination between connectivity maintenance and efficient exploration, rather than treating the two modules separately.

\addtolength{\textheight}{-12cm}   % This command serves to balance the column lengths
                                  % on the last page of the document manually. It shortens
                                  % the textheight of the last page by a suitable amount.
                                  % This command does not take effect until the next page
                                  % so it should come on the page before the last. Make
                                  % sure that you do not shorten the textheight too much.

%%%%%%%%%%%%%%%%%%%%%%%%%%%%%%%%%%%%%%%%%%%%%%%%%%%%%%%%%%%%%%%%%%%%%%%%%%%%%%%%

%%%%%%%%%%%%%%%%%%%%%%%%%%%%%%%%%%%%%%%%%%%%%%%%%%%%%%%%%%%%%%%%%%%%%%%%%%%%%%%%

%%%%%%%%%%%%%%%%%%%%%%%%%%%%%%%%%%%%%%%%%%%%%%%%%%%%%%%%%%%%%%%%%%%%%%%%%%%%%%%%

\bibliographystyle{ieeetr} %ieeetr国际电气电子工程师协会期刊
\bibliography{reference} % ref就是之前建立的ref.bib文件的前缀

\setlength{\textheight}{23.5cm}
\newpage
\section*{Appendix}

\subsection{Formulation of Solving Exact LoS-Distance}

To obtain the accurate LoS-distance $d^{\text{los}}_{ji}$ as in Def.~\ref{def_los_distance}, we have to solve the following nonlinear optimization problem:
\begin{equation}
    \min_{\p} \Vert \p - \q_{i} \Vert_2 ,
\end{equation}
\begin{equation*}
\begin{aligned}
    \text{s.t.}~& i)~\p = f(\p'); \quad ii)~\sum\nolimits_{k} z_{k} = 1, z_{k} \in \{0, 1\};\\
    &iii)~\p' = \sum\nolimits_{k}z_{k}(\alpha \boldsymbol{a}_{k} + (1-\alpha) \boldsymbol{a}_{k+1}), \alpha \in [0, 1].
\end{aligned}
\end{equation*}
The above constraints (i-iii) restrict that the point $\p$ is located at the boundary of $\mathcal{W}_j$, with a corresponding convex hull $\mathcal{H}_j$ with $K$ edges. 
% The exact LoS-distance can then be obtained as $d^{\text{los}}_{ji} = \Vert \p^{*} - \q_{ji} \Vert_2$, where $\p^{*}$ is the optimal value of $\p$, which corresponds to the closest point on the boundary of $\mathcal{W}_j$ to $\q_i$.
Finding the optimal $\p^{*}$ involves solving a high-order equation that generally does not have an analytical solution.
In this work, we solve the above non-linear program using Gurobi. 
The time efficiency of solving the problem in 2D case is compared in Tab.~\ref{tab_approx_los_dist}.

\subsection{Gradient of Approximated LoS-Distance}
Following the definition of $\tilde{d}_{k}$ in Eq.~(\ref{def_dk}), the gradient of $\tilde{d}_{k}$ is derived as
\begin{equation}
    \frac{\partial\tilde{d}_{k}}{\partial\q_{i}} = 
    \begin{cases} 
    \frac{\q_{i} \boldsymbol{a}_{k}}{\Vert \q_{i}\boldsymbol{a}_{k} \Vert},  \quad\quad
    \frac{(\q_{i}\boldsymbol{a}_{k})^{\top}(\boldsymbol{a}_{k+1}\boldsymbol{a}_{k})}{\Vert \boldsymbol{a}_{k+1}\boldsymbol{a}_{k}\Vert} < 0, \\
        \frac{\q_{i}\boldsymbol{a}_{k+1}}{\Vert \q_{i}\boldsymbol{a}_{k+1} \Vert},  \quad \frac{(\q_{i}\boldsymbol{a}_{k})^{\top}(\boldsymbol{a}_{k+1}\boldsymbol{a}_{k})}{\Vert \boldsymbol{a}_{k+1}\boldsymbol{a}_{k}\Vert} > 1, \\
        \frac{(\q_{i}\boldsymbol{a}_{k})}{\tilde{d}_{k}} -\frac{(\boldsymbol{a}_{k+1} \boldsymbol{a}_{k})^{\top}(\q_{i}\boldsymbol{a}_{k})(\boldsymbol{a}_{k+1} \boldsymbol{a}_{k})}{\tilde{d}_{k}\Vert \boldsymbol{a}_{k+1}\boldsymbol{a}_{k}\Vert^{2}}, \quad\text{otherwise}.
    \end{cases}
\end{equation}
The gradient of the approximated LoS-distance can push a robot away from the closest boundary of the neighboring robot's visible region, thereby maintaining LoS-connectivity.

% , while we need a differentiable expression of $d^{\text{los}}_{ji}$ to derive the connectivity velocity.
% Moreover, its high time complexity prevents it from real-time applications, as compared in Tab.~\ref{tab_approx_los_dist}.
% Therefore, we resort to an effective and safe approximation of the visible region that supports real-time LoS-distance evaluation with differentiable expressions, which will be introduced in the next subsection.

\subsection{Proof of Proposition~\ref{prop_subset}}
\label{sec_prove_of_prop2}

\begin{figure}[!h]
\centering
\includegraphics[width=.9\linewidth]{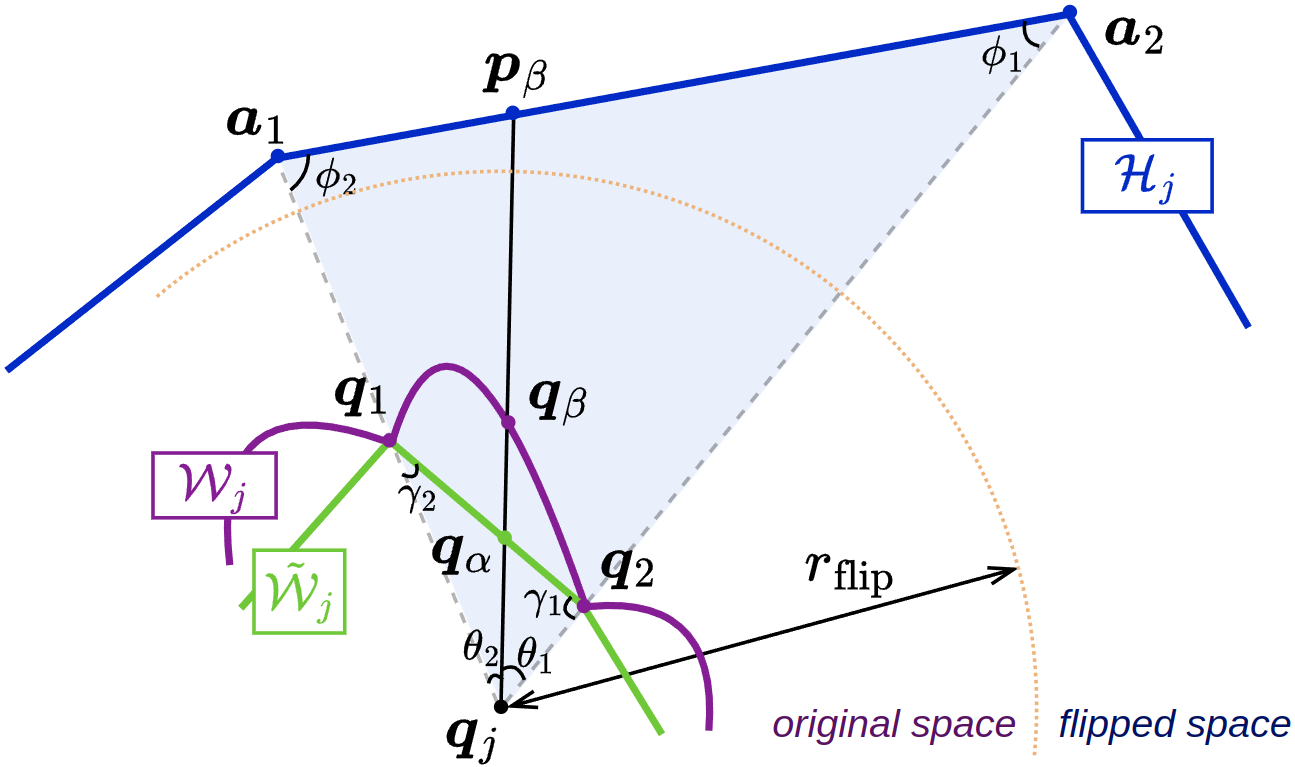}
\caption{Illustration of notations used in the proof of Prop.~\ref{prop_subset}.}
\label{fig_proof}
\end{figure}

\begin{proof}
As shown in Fig.~\ref{fig_proof}, the boundary of $\mathcal{W}_j$ is described by several segments of curves, each of which (\eg, $\widehat{\q_{1}\q_{2}}$) is obtained by flipping one edge (\eg, $\langle\boldsymbol{a}_{1},\boldsymbol{a}_{2}\rangle$) of the convex hull $\mathcal{H}_j$ from the flipped space back to the original space.
To proof $\tilde{\mathcal{W}}_j \subseteq \mathcal{W}_j$, it is equivalent to proof that $\forall k\in [1, 2, ..., K]$, $(\tilde{\mathcal{W}}_j \cap \triangle \q_{j}\boldsymbol{a}_{k}\boldsymbol{a}_{k+1} )\subseteq (\mathcal{W}_j \cap \triangle \q_{j}\boldsymbol{a}_{k}\boldsymbol{a}_{k+1})$, where $\triangle \q_{j}\boldsymbol{a}_{k}\boldsymbol{a}_{k+1}$ denotes the triangle formed by three points $(\q_{j},\boldsymbol{a}_{k},\boldsymbol{a}_{k+1})$, and $\langle\boldsymbol{a}_{k},\boldsymbol{a}_{k+1} \rangle$ is the $k$-th edge of $\mathcal{H}_j$.

We takes the edge $\langle\boldsymbol{a}_{1},\boldsymbol{a}_{2}\rangle$ of $\mathcal{H}_j$ as an example, as shown in Fig.~\ref{fig_proof}. 
The corresponding boundaries of $\mathcal{W}_j$ and $\tilde{\mathcal{W}}_j$ are denoted as a curve $\widehat{\q_{1}\q_{2}}$ and a line segment $\overline{\q_1\q_2}$ respectively. 
By construction, it holds that $\boldsymbol{a}_1 = f(\q_{1})$ and $\boldsymbol{a}_2 = f(\q_{2})$, where $f(\cdot)$ denotes the flipping operation.
We take a random point $\q_{\alpha} = \q_{2} + \alpha (\q_1 - \q_2)$ with $\alpha \in [0, 1]$ on $\overline{\q_1\q_2}$. 
The ray from $\q_j$ to $\q_{\alpha}$ will intersect with $\widehat{\q_{1}\q_{2}}$ at $\q_{\beta}$, and will also intersect with $\overline{\boldsymbol{a}_1\boldsymbol{a}_2}$ at $\boldsymbol{a}_{\beta}$, where $\boldsymbol{a}_{\beta}$ can be defined as $\boldsymbol{a}_{\beta} = \boldsymbol{a}_{2} + \beta (\boldsymbol{a}_1 - \boldsymbol{a}_2)$ with $\beta \in [0, 1]$. 
Note that $\beta$ can be different from $\alpha$.

By definition, both $\mathcal{W}_j$ and $\tilde{\mathcal{W}}_j$ are stat-convex polytopes with $\q_j$ being the center, $\ie$, for any point $\q \in \mathcal{W}_j$, the line segment $\overline{\q_j\q}$ lies entirely within $\mathcal{W}_j$.
The sample property also holds for $\tilde{\mathcal{W}}_j$.
% As we approximate the boundary of the visible region as line segments, as shown in Fig.~\ref{fig_proof}, any point $\q_{\beta}$ on the boundary $\widehat{\q_{1}\q_{2}}$ of $\mathcal{W}_j$ corresponds to a point $\q_{\alpha}$ on the line segment $\overline{\q_{1}\q_{2}}$ as a boundary of $\tilde{\mathcal{W}}_j$, with $\overline{\q_{j}\q_{\alpha}}$ and $\overline{\q_{j}\q_{\beta}}$ being colinear.
Therefore, to prove $\tilde{\mathcal{W}}_j\subseteq \mathcal{W}_j$, it is sufficient to prove that for any point $\q_{\alpha}$ on the boundary $\overline{\q_1\q_2}$, and the corresponding point $\q_{\beta}$ on the boundary $\widehat{\q_{1}\q_{2}}$, it holds that $\Vert\q_{j}\q_{\alpha}\Vert \le \Vert\q_{j}\q_{\beta}\Vert$.
Otherwise, if $\Vert\q_{j}\q_{\alpha}\Vert > \Vert\q_{j}\q_{\beta}\Vert$, the line segment $\overline{\q_{\alpha}\q_{\beta}}\subseteq \tilde{\mathcal{W}}_j$ and $\overline{\q_{\alpha}\q_{\beta}}\nsubseteq \mathcal{W}_j$, which contradicts with $\tilde{\mathcal{W}}_j\subseteq \mathcal{W}_j$.
In the following, we will prove that $\Vert\q_{j}\q_{\alpha}\Vert \le \Vert\q_{j}\q_{\beta}\Vert$, $\forall \alpha\in [0, 1]$.

First, by applying the Law of Sines in $\triangle \q_{j}\q_{\alpha}\q_2$, it holds that
\begin{equation}
    \frac{\alpha \cdot \Vert \q_{1}\q_{2}\Vert}{\sin{\theta_1}} = \frac{\Vert\q_{j}\q_{\alpha}\Vert}{\sin{\gamma_1}},
\label{eq_p1}
\end{equation}
and 
\begin{equation}
    \frac{(1-\alpha) \cdot \Vert\q_{1}\q_{2}\Vert}{\sin{\theta_2}} = \frac{\Vert\q_{j}\q_{\alpha}\Vert}{\sin{\gamma_2}}.
\label{eq_p2}
\end{equation}
Also in $\triangle \q_{j}\q_{1}\q_{2}$, it holds that 
\begin{equation}
    \frac{\Vert\q_{j}\q_1\Vert}{\sin{\gamma_1}} = \frac{\Vert\q_{j}\q_2\Vert}{\sin{\gamma_2}} = \frac{\Vert\q_{1}\q_2\Vert}{\sin{\theta}},
\label{eq_p3}
\end{equation}
where $\theta = \theta_1 + \theta_2$. 
By eliminating $\sin{\gamma_1}$ and $\sin{\gamma_2}$ in Eq.~(\ref{eq_p1}) and (\ref{eq_p2}), $\alpha$ can be expressed as
\begin{equation}
    \alpha = \frac{\sin{\theta_1}\cdot \Vert\q_{j}\q_{2}\Vert}{\sin{\theta_2}\cdot\Vert\q_{j}\q_{1}\Vert + \sin{\theta_1}\cdot\Vert\q_{j}\q_{2}\Vert}.
\label{eq_p4}
\end{equation}
Similarly, 
\begin{equation}
        \beta = \frac{\sin{\theta_1}\cdot \Vert\q_{j}\boldsymbol{a}_{2}\Vert}{\sin{\theta_2}\cdot\Vert\q_{j}\boldsymbol{a}_{1}\Vert + \sin{\theta_1}\cdot\Vert\q_{j}\boldsymbol{a}_{2}\Vert}.
\label{eq_p5}
\end{equation}
Combining Eq.~(\ref{eq_p1}) and Eq.~(\ref{eq_p3}), it holds that
\begin{equation}
    \Vert\q_{j}\q_{\alpha}\Vert = \frac{\alpha\cdot \Vert\q_{j}\q_{1}\Vert\cdot \sin{\theta}}{\sin{\theta_1}}.
\end{equation}
Similarly, 
\begin{equation}
        \Vert\q_{j}\boldsymbol{a}_{\beta}\Vert = \frac{\beta\cdot \Vert\q_{j}\boldsymbol{a}_{1}\Vert\cdot \sin{\theta}}{\sin{\theta_1}}.
\end{equation}
To prove $\Vert\q_{j}\q_{\alpha}\Vert \le \Vert\q_{j}\q_{\beta}\Vert$, it is equivalent to prove that 
\begin{equation}
    \begin{aligned}
        &
        2r_{\text{flip}} - \Vert\q_{j}\q_{\alpha}\Vert \ge \Vert\q_{j}\boldsymbol{a}_{\beta}\Vert\\
        \Rightarrow
        &
        2r_{\text{flip}}\sin{\theta_1} - \alpha\cdot \Vert\q_{j}\q_{1}\Vert \cdot \sin{\theta} \ge \beta \cdot \Vert\q_{j}\boldsymbol{a}_{1}\Vert \sin{\theta}\\
        \Rightarrow
        &
        2r_{\text{flip}} \ge \frac{\sin{\theta}}{\sin{\theta_1}}\cdot\alpha\cdot \Vert\q_{j}\q_{1}\Vert + \frac{\sin{\theta}}{\sin{\theta_1}}\cdot\beta \cdot \Vert\q_{j}\boldsymbol{a}_{1}\Vert = RHS
    \end{aligned}
\label{eq_p6}
\end{equation}
Taking Eq.~(\ref{eq_p4}) and (\ref{eq_p5}) into Eq.~(\ref{eq_p6}), the righ-hand-side (RHS) of Eq.~(\ref{eq_p6}) can be expanded as
\begin{equation}
\begin{aligned}
        RHS = &\frac{\sin{\theta}\cdot \Vert\q_{j}\q_{1}\Vert\cdot \Vert\q_{j}\q_{2}\Vert }{\sin{\theta_2}\cdot\Vert\q_{j}\q_{1}\Vert + \sin{\theta_1}\cdot\Vert\q_{j}\q_{2}\Vert} + \\
        &
        \frac{\sin{\theta}\cdot \Vert\q_{j}\boldsymbol{a}_{1}\Vert\cdot \Vert\q_{j}\boldsymbol{a}_{2}\Vert }{\sin{\theta_2}\cdot\Vert\q_{j}\boldsymbol{a}_{1}\Vert + \sin{\theta_1}\cdot\Vert\q_{j}\boldsymbol{a}_{2}\Vert}\\
        = &\frac{\Vert\q_{j}\q_{1}\Vert\cdot \Vert\q_{j}\q_{2} \Vert}{\frac{\sin{\theta_2}}{\sin{\theta}}\cdot\Vert\q_{j}\q_{1}\Vert + \frac{\sin{\theta_1}}{\sin{\theta}}\cdot\Vert\q_{j}\q_{2}\Vert} + \\
        &
        \frac{\Vert\q_{j}\boldsymbol{a}_{1}\Vert\cdot \Vert\q_{j}\boldsymbol{a}_{2}\Vert }{\frac{\sin{\theta_2}}{\sin{\theta}}\cdot\Vert\q_{j}\boldsymbol{a}_{1}\Vert + \frac{\sin{\theta_1}}{\sin{\theta}}\cdot\Vert\q_{j}\boldsymbol{a}_{2}\Vert}
\end{aligned}
\label{eq_p7}
\end{equation}
The function $h_{\theta}(x) = \frac{\sin{x}}{\sin{\theta}} \ge \frac{x}{\theta}$ when $x \in [0, \theta]$, and $\frac{\sin{x}}{\sin{\theta}} = \frac{x}{\theta}$ only when $x = 0$ or $x = \theta$.
Therefore, the RHS of Eq.~(\ref{eq_p7}) satisfies
\begin{equation}
\small
    RHS \le 
    \frac{\Vert\q_{j}\q_{1}\Vert\cdot \Vert\q_{j}\q_{2}\Vert }{\frac{\theta2}{\theta}\Vert\q_{j}\q_{1}\Vert + \frac{\theta_1}{\theta}\Vert\q_{j}\q_{2}\Vert} + 
        \frac{\Vert\q_{j}\boldsymbol{a}_{1}\Vert\cdot \Vert\q_{j}\boldsymbol{a}_{2}\Vert }{\frac{\theta_2}{\theta}\Vert\q_{j}\boldsymbol{a}_{1}\Vert + \frac{\theta_1}{\theta}\Vert\q_{j}\boldsymbol{a}_{2}\Vert}.
\label{eq_rhs_le}
\end{equation}
As $\theta_1$, $\theta_2$, $\theta> 0$, and $\theta_1 + \theta_2 = \theta$, we define $\eta = \frac{\theta_2}{\theta} \in [0, 1]$ to replace  $\theta_1$, $\theta_2$ and $\theta$. 
The Eq.~(\ref{eq_rhs_le}) can be written as 
\begin{equation}
\small
        RHS \le 
    \frac{\Vert\q_{j}\q_{1}\Vert\cdot \Vert\q_{j}\q_{2}\Vert }{\eta\Vert\q_{j}\q_{1}\Vert + (1-\eta)\Vert\q_{j}\q_{2}\Vert} + 
        \frac{\Vert\q_{j}\boldsymbol{a}_{1}\Vert\cdot \Vert\q_{j}\boldsymbol{a}_{2}\Vert }{\eta\Vert\q_{j}\boldsymbol{a}_{1}\Vert + (1-\eta)\Vert\q_{j}\boldsymbol{a}_{2}\Vert}.
\label{eq_p8}
\end{equation}
Note that $\Vert\q_{j}\boldsymbol{a}_{1}\Vert = 2r_{\text{flip}} - \Vert\q_{j}\q_{1}\Vert$, $\Vert\q_{j}\boldsymbol{a}_{2}\Vert = 2r_{\text{flip}} - \Vert\q_{j}\q_{2}\Vert$, we have $\eta\cdot\Vert\q_{j}\boldsymbol{a}_{1}\Vert + (1-\eta)\cdot\Vert\q_{j}\boldsymbol{a}_{2}\Vert = 2r_{\text{flip}} - (\eta\cdot\Vert\q_{j}\q_{1}\Vert + (1-\eta)\cdot\Vert\q_{j}\q_{2}\Vert)$.
Let $x$ be an auxiliary variable defined as $x = \eta\cdot\Vert\q_{j}\q_{1}\Vert + (1-\eta)\cdot\Vert\q_{j}\q_{2}\Vert$, which takes values from $\min\{\Vert\q_{j}\q_{1}\Vert, \Vert\q_{j}\q_{2}\Vert\}$ to $\max\{\Vert\q_{j}\q_{1}\Vert, \Vert\q_{j}\q_{2}\Vert\}$.
The Eq.~(\ref{eq_p8}) can be rewritten as 
\begin{equation}
    RHS \le \frac{\Vert\q_{j}\q_{1}\Vert\cdot \Vert\q_{j}\q_{2} \Vert}{x} + 
        \frac{\Vert\q_{j}\boldsymbol{a}_{1}\Vert\cdot \Vert\q_{j}\boldsymbol{a}_{2}\Vert }{2r_{\text{flip}} - x} = g(x).
\end{equation}
The derivative of $g(x)$ is calculated as
\begin{equation}
\begin{aligned}
    \frac{\partial g(x)}{\partial x} 
    &= \frac{-\Vert\q_{j}\q_{1}\Vert\cdot \Vert\q_{j}\q_{2}\Vert}{x^2} + \frac{-\Vert\q_{j}\boldsymbol{a}_{1}\Vert\cdot \Vert\q_{j}\boldsymbol{a}_{2}\Vert}{(2r_{\text{flip}} - x)^2}\\
    &
    =\frac{Ax^2 + Bx - C}{x^{2}(2r_{\text{flip}} -x)^2},
\end{aligned}
\end{equation}
where $A = \Vert\q_{j}\boldsymbol{a}_{1}\Vert\cdot \Vert\q_{j}\boldsymbol{a}_{2}\Vert - \Vert\q_{j}\q_{1}\Vert\cdot \Vert\q_{j}\q_{2}\Vert$, $B = 4r_{\text{flip}}\cdot\Vert\q_{j}\q_{1}\Vert\cdot \Vert\q_{j}\q_{2}\Vert$, $C = 4r_{\text{flip}}^{2}\cdot \Vert\q_{j}\q_{1}\Vert\cdot \Vert\q_{j}\q_{2}\Vert$.
The sign of $\frac{\partial g(x)}{\partial x}$ depends on the quadratic function $Ax^2 + Bx - C$.
Since $\Vert\q_{j}\boldsymbol{a}_{1}\Vert> r_{\text{flip}}$, $\Vert\q_{j}\q_{1}\Vert< r_{\text{flip}}$, we have $\Vert\q_{j}\boldsymbol{a}_{1}\Vert > \Vert\q_{j}\q_{1}\Vert$. Similarly, it holds that $\Vert\q_{j}\boldsymbol{a}_{2}\Vert > \Vert\q_{j}\q_{2}\Vert$.
Therefore, $A > 0$; the axis of symmetry is $\frac{-B}{2A} < 0$; and the function has two roots because the discriminant $B^2 + 4AC > 0$.
Therefore, according to the properties of a univariate quadratic curve, the function $g(x)$ has two possible trends within the domain of $x$: (1) $g(x)$ decreases first and then increases; (2) $g(x)$ increases monotonically.
In both cases, the maximum value of $g(x)$ is achieved on the boundary of the domain of $x$.
When $x = \Vert\q_{j}\q_1\Vert$, $g(x) = 2r_{\text{flip}}$; and when $x = \Vert\q_{j}\q_2\Vert$, $g(x) = 2r_{\text{flip}}$.
Therefore, it holds that
\begin{equation}
    RHS \le g(x) \le 2r_{\text{flip}},
\end{equation}
which proves Eq.~\ref{eq_p6}.
Consequently, $\Vert\q_{j}\q_{\alpha}\Vert \le \Vert\q_{j}\q_{\beta}\Vert$, $\forall \alpha\in [0, 1]$.
And for a convex hull $\mathcal{H}_j$ with $K$ edges, it holds that $(\tilde{\mathcal{W}}_j \cap \triangle \q_{j}\boldsymbol{a}_{k}\boldsymbol{a}_{k+1} )\subseteq (\mathcal{W}_j \cap \triangle \q_{j}\boldsymbol{a}_{k}\boldsymbol{a}_{k+1})$, $\forall k\in [1, ..., K]$.
This concludes the proof of $\tilde{\mathcal{W}}_j \subseteq \mathcal{W}_j$.
\end{proof}

With Prop.~\ref{prop_subset}, we have the following corollary:
\begin{corollary}
    $\forall k\in [1, \cdots, K]$, the region $\mathcal{W}_j \cap \triangle \q_{j}\boldsymbol{a}_{k}\boldsymbol{a}_{k+1}$ is convex, where $\triangle \q_{j}\boldsymbol{a}_{k}\boldsymbol{a}_{k+1}$ denotes the triangle formed by three points $\langle\q_{j},\boldsymbol{a}_{k},\boldsymbol{a}_{k+1}\rangle$.
\label{prop_convex}
\end{corollary}

\begin{figure}[!h]
\centering
\includegraphics[width=.9\linewidth]{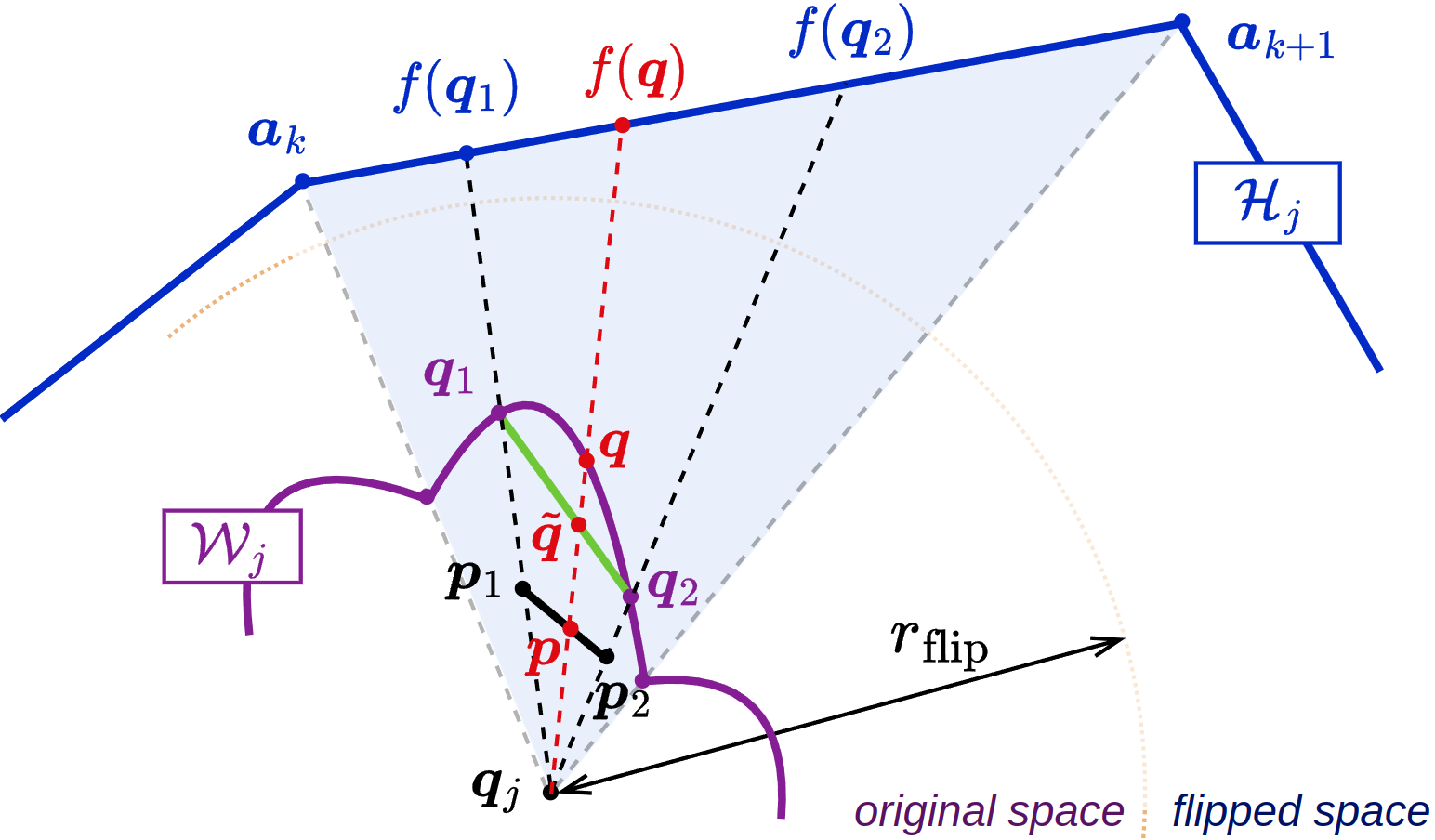}
\caption{Illustration of notations used in the proof of Cor.~\ref{prop_convex}.}
\label{fig_proof2}
\end{figure}

\begin{proof}
Let $\p_{1}$, $\p_{2}$ be two random points within $\mathcal{W}_j \cap \triangle \q_{j}\boldsymbol{a}_{k}\boldsymbol{a}_{k+1}$, as shown in Fig.~\ref{fig_proof2}.
If $\forall \alpha\in [0, 1]$, $\alpha\p_{1} + (1 - \alpha)\p_{2} \in \mathcal{W}_j \cap \triangle \q_{j}\boldsymbol{a}_{k}\boldsymbol{a}_{k+1}$, then the region $\mathcal{W}_j \cap \triangle \q_{j}\boldsymbol{a}_{k}\boldsymbol{a}_{k+1}$ is convex.

By construction, the region $\mathcal{W}_j \cap \triangle \q_{j}\boldsymbol{a}_{k}\boldsymbol{a}_{k+1}$ is star-convex. 
Therefore, the ray originating from $\p_j$ to $\p_1$, $\q_2$ will intersect with the boundary of $\mathcal{W}_j$ at $\q_1$ and $\q_2$, respectively; and it holds that $\Vert \p_1 \Vert \le \Vert \q_{1}\Vert$ and $\Vert \p_2 \Vert \le \Vert \q_{2}\Vert$.
Let $\p = \alpha\p_{1} + (1 - \alpha)\p_{2}$ be a randomly selected point on the line segment $\overline{\p_{1}\p_{2}}$, and the ray from $\q_j$ to $\p$ intersect with $\overline{\q_{1}\q_{2}}$, $\mathcal{W}_j$ and $\overline{\boldsymbol{a}_{k}\boldsymbol{a}_{k+1}}$ at $\tilde{\q}$, $\q$ and $f(\q)$ respectively, as shown in Fig.~\ref{fig_proof2}.
If $\Vert f(\p)\Vert > \Vert f(\q)\Vert$, it holds that $\p\in \mathcal{W}_j \cap \triangle \q_{j}\boldsymbol{a}_{k}\boldsymbol{a}_{k+1}$ according to visibility determination as in Prop.~\ref{prop_visible_determiniation}, which then proves Cor.~\ref{prop_convex}.

Recall that $f(\cdot)$ is the spherical flipping function.
It holds that 
\begin{equation}
    \Vert f(\p) \Vert = 2r_{\text{flip}} - \Vert \p \Vert
    \ge
    2r_{\text{flip}} - \Vert \tilde{\q} \Vert = \Vert f(\tilde{\q})\Vert.
\end{equation}
Note that the second inequality holds because $\p\in \triangle\q_{j}\q_{1}\q_{2}$, and $\triangle\q_{j}\q_{1}\q_{2}$ is also star-convex.
We have proved similarly in the proof of Prop.~\ref{prop_subset} that $\Vert f(\tilde{\q}) \Vert \ge \Vert f(\q) \Vert$. 
Therefore, $\Vert f(\p) \Vert \ge \Vert f(\q) \Vert$; and $\p\in \mathcal{W}_j \cap \triangle \q_{j}\boldsymbol{a}_{k}\boldsymbol{a}_{k+1}$, $\forall \alpha \in [0, 1]$.
This proves that the region $\mathcal{W}_j \cap \triangle \q_{j}\boldsymbol{a}_{k}\boldsymbol{a}_{k+1}$ is convex.
\end{proof}

\subsection{Proof of Proposition~\ref{prop_lower_bound}}

\begin{proof}
We prove Prop.~\ref{prop_lower_bound} by contradiction. 
Assume that the approximated LoS-distance $\tilde{d}^{\text{los}}_{ji} > d^{\text{los}}_{ji}$, the circle $\tilde{\mathcal{S}}$ (or shpere in 3D) centered at robot $i$ with radius $\tilde{d}^{\text{los}}_{ji}$ will cover a larger area than the circle $\mathcal{S}$ with radius $d^{\text{los}}_{ji}$, \ie, $\mathcal{S} \subset \tilde{\mathcal{S}}\subseteq \tilde{\mathcal{W}}_j$.
According to Prop.~\ref{prop_subset}, $\tilde{\mathcal{W}}_j\subseteq \mathcal{W}_j$, it holds that $ \tilde{\mathcal{S}}\subseteq \tilde{\mathcal{W}}_j\subseteq \mathcal{W}_j$.
Therefore, the distance from $\q_j$ to the boundary of $\mathcal{W}_j$ is not smaller than $\tilde{d}^{\text{los}}_{ji}$, \ie, $d^{\text{los}}_{ji} \ge\tilde{d}^{\text{los}}_{ji}$, which contradicts with $\tilde{d}^{\text{los}}_{ji} > d^{\text{los}}_{ji}$.
This proves Prop.~\ref{prop_lower_bound}.
\end{proof}

\end{document}